\pdfoutput=1
\documentclass[reqno]{amsart}
\usepackage[utf8]{inputenc}
\usepackage[T1]{fontenc}
\usepackage[margin={4cm,2cm}]{geometry}
\usepackage[foot]{amsaddr}
\numberwithin{equation}{section}
\usepackage{ntrees}

\usepackage{hyperref, cleveref}
\usepackage{bookmark}
\usepackage[ocgcolorlinks]{ocgx2} 
\usepackage{xcolor}
\usepackage{tikz-cd}
\usepackage{mathtools, lmodern, microtype, mathabx, stmaryrd, shuffle, amssymb}
\usepackage{mleftright, xspace}
\usepackage{fixcmex}
\usepackage{mathrsfs}
\usepackage[sort,compress]{cite}
\mleftright

\usepackage{todonotes}
\hypersetup{citecolor=blue,%
pdfinfo={Author={P. Friz, P. Hager, and N. Tapia}, Title={Unified signature cumulants and generalized Magnus expansions},%
Keywords={Signatures, L\'evy processes, moment-cumulant relations, characteristic functions, diamond product},%
Subject={MSC(2020) 60L10, 60L90, 60E10, 60G44, 60G48, 60G51, 60J76}}
}
\newtheorem{theorem}{Theorem}[section]
\newtheorem{proposition}[theorem]{Proposition}
\newtheorem{lemma}[theorem]{Lemma}
\newtheorem{corollary}[theorem]{Corollary}
\newtheorem{remark}[theorem]{Remark}
\theoremstyle{definition}
\newtheorem{definition}[theorem]{Definition}

\newtheorem{xmp}[theorem]{Example}

\newenvironment{example}{\pushQED{\qed}\begin{xmp}}{\popQED\end{xmp}}

\DeclareFontFamily{U}{mathb}{}
\DeclareFontShape{U}{mathb}{m}{n}{<-5.5> mathb5 <5.5-6.5> mathb6 
<6.5-7.5> mathb7 <7.5-8.5> mathb8 <8.5-9.5> mathb9 <9.5-11> mathb10 
<11-> mathb12}{}
\DeclareSymbolFont{mathb}{U}{mathb}{m}{n}
\DeclareFontSubstitution{U}{mathb}{m}{n}

\DeclareMathSymbol{\blackdiamond}{\mathbin}{mathb}{"0C}
\DeclareFontFamily{U}{shuffle}{}
\DeclareFontShape{U}{shuffle}{m}{n}{ <-8>shuffle7 <8->shuffle10}{}

\DeclareMathOperator{\ad}{ad}

\newcommand{\etfrak}{\mathfrak{y}}

\def\bA{\mathbf{A}}
\def\bB{\mathbf{B}}

\def\bF{\mathbf{F}}

\def\bL{\mathbf{L}}
\def\bM{\mathbf{M}}

\def\bS{\mathbf{S}}

\def\bX{\mathbf{X}}
\def\bY{\mathbf{Y}}
\def\bZ{\mathbf{Z}}

\def\bs{\mathbf{s}}

\def\Sesup{\mathscr{S}}
\def\indik{\mathbf 1}
\def\dparl{\mathopen{(\mkern-3mu(}}
\def\dparr{\mathclose{)\mkern-3mu)}}

\newcommand{\R}{\mathbb{R}}
\newcommand{\Rd}{{\mathbb{R}^d}}
\newcommand{\F}{\mathcal{F}}
\newcommand{\G}{\mathcal{G}}
\newcommand{\PM}{\mathbb{P}}
\newcommand{\E}{\mathbb{E}}
\newcommand{\bx}{\mathbf{x}}
\newcommand{\by}{\mathbf{y}}
\newcommand{\N}{{\mathbb{N}}}
\newcommand{\None}{{\mathbb{N}_{\ge1}}}

\newcommand{\Lcal}{\mathcal{L}}

\newcommand{\Id}{\mathrm{Id}}
\newcommand{\I}{\mathbf{I}}
\newcommand{\W}{\mathcal{W}}

\newcommand{\dd}{\mathrm{d}}
\newcommand{\TT}{\mathcal{T}}

\newcommand{\Lie}[2]{\left[ #1, #2 \right]}
\newcommand{\tf}{\TT_0}
\newcommand{\ideal}{\mathcal{I}}

\DeclarePairedDelimiter{\cvbracket}{\langle\mkern-4mu\langle}{\rangle\mkern-4mu\rangle}
\newcommand{\outerbracket}[2]{\cvbracket{#1, #2}}

\newcommand{\innerbracket}[2]{\left\langle #1, #2 \right\rangle}
\newcommand{\innerbracketsmall}[2]{\langle #1, #2 \rangle}

\newcommand{\QV}[1]{\left\langle #1 \right\rangle}

\newcommand{\QVsmall}[1]{\langle #1 \rangle}
\newcommand{\CV}[2]{\left\langle #1, #2 \right\rangle}

\DeclarePairedDelimiter{\abss}{\vert}{\vert}
\def\abs{\abss*}
\DeclarePairedDelimiter{\Abss}{\Vert}{\Vert}
\def\Abs{\Abss*}

\DeclarePairedDelimiter{\Aabss}{\vert\mkern-2.5mu\vert\mkern-2.5mu\vert}{\vert\mkern-2.5mu\vert\mkern-2.5mu\vert}
\def\Aabs{\Aabss*}

\newcommand{\pisym}{\pi_{\mathrm{Sym}}}
\newcommand{\kap}{\pmb{\kappa}}
\newcommand{\kapT}[1]{\pmb{\kappa}_{#1}(T)}

\newcommand{\esig}{\pmb{\mu}}

\newcommand{\KK}{\mathbb{K}}

\newcommand{\fmu}{\pmb{\mu}}

\newcommand{\Qua}{\mathrm{Qua}}
\newcommand{\Cov}{\mathrm{Cov}}

\newcommand{\Sig}{\mathrm{Sig}}
\newcommand{\T}{\mathcal{T}}
\newcommand{\Sy}{\mathcal{S}} %

\newcommand{\hatexp}{\widehat \exp}
\newcommand{\hatlog}{\widehat \log}

\newcommand{\Se}{\mathscr{S}} %
\newcommand{\D}{\mathscr{C}} %
\newcommand{\Sec}{\Se^c} %
\newcommand{\Ma}{\mathscr{M}} %
\def\cadlag{càdlàg\xspace}
\newcommand{\loc}{\mathrm{loc}}
\newcommand{\Fv}{\mathscr{V}} %
\newcommand{\var}{\mathrm{var}}
\newcommand{\HSe}{\mathscr{H}}
\newcommand{\HSehom}{\mathscr{H}}
\newcommand{\HSehomSym}{\widehat{\mathscr{H}}}
\newcommand{\homnqN}[1]{\Aabs{#1}_{\HSehom^{q,N}}}

\newcommand{\homnqNs}[1]{\Aabss{#1}_{\HSehom^{q,N}}}
\newcommand{\homns}[1]{\Aabss{#1}_{\HSehom^{1,N}}}

\begin{document}
\title[On expected signatures and signature cumulants in semimartingale models {\ \ }]{On expected signatures and signature cumulants in semimartingale models}
\author[P.~Friz]{Peter~K.~Friz$^{\dagger,\ddagger}$}
\author[Paul~P. Hager]{Paul~P.~Hager$^\ast$}
\author[N.~Tapia]{Nikolas Tapia$^{\ddagger}$}
\email{friz@math.tu-berlin.de}\email{paul.peter.hager@univie.ac.at}\email{tapia@wias-berlin.de}
\address{$^\dagger$Institut für Mathematik, TU Berlin, Str. des 17. Juni 136, 10586 Berlin, Germany.}
\address{$^\ddagger$Weierstrass Institute, Mohrenstr. 39, 10117 Berlin, Germany.}
\address{$^\ast$Department of Statistics and Operations Research, University of Vienna, Kolingasse 14-16, 1090 Wien, Austria.}

\begin{abstract}
The signature transform, a Cartan type development, translates paths into high-dimensional feature vectors, capturing their intrinsic characteristics. Under natural conditions, the expectation of the signature determines the law of the signature, providing a statistical summary of the data distribution. This property facilitates robust modeling and inference in machine learning and stochastic processes. Building on previous work by the present authors [{\em Unified signature cumulants and generalized Magnus expansions, FoM Sigma '22}] we here revisit the actual computation of expected signatures, in a general semimartingale setting. Several new formulae are given. A log-transform of (expected) signatures leads to log-signatures (signature cumulants), offering a significant reduction in complexity.
\end{abstract}

\maketitle

\tableofcontents

\clearpage

\section{Introduction}

Signatures and expected signatures have become increasingly important in recent years, offering a top-down description of sequential data and its statistics, respectively.
To formalize this recall the tensor series over $\mathbb{R}^d$, denoted $\T \coloneq 
T\dparl\mathbb{R}^d\dparr$, defined as infinite direct product of tensor powers of $\mathbb{R}^d$, with elements of the form
\[
\bx = (\bx^{(0)}, \bx^{(1)}, \bx^{(2)}, \dotsc) \equiv \bx^{(0)} + \bx^{(1)} + \bx^{(2)} + \dotsb, 
\]
$\bx^{(k)} \in (\mathbb{R}^d)^{\otimes k}, k \in \N$, with both $(\mathbb{R}^d)^{\otimes 0} \cong \R$ and $(\mathbb{R}^d)^{\otimes 1} \cong \mathbb{R}^d$ embedded in $\T$. Tensor series can be multiplied, so that $\T$ carries a natural algebra structure, $\exp$ and $\log$ are then defined by their usual power series.
Given a smooth $\mathbb{R}^d$-valued path $X$, its {\em signature} on $[0,T]$,  is obtained from solving the universal linear differential equation
\begin{equation} \label{equ:SigSmooth}
     \dot \bS = \bS\,\dot X, \qquad \bS_0 = 1 \equiv (1,0,0,...) \in \T_1 	,
\end{equation}
which just rephrases the common definition in terms of iterated integrals. We then set $\Sig (X)_{0,T} := \bS_T$.
Here and below, write $\T_\lambda$ for series starting with (scalar) $\lambda$. 
It is known, essentially as consequence of the chain rule, that the log-signature $\log \bS_T$ takes values in $\mathcal{L} \equiv \mathrm{Lie} \dparl\mathbb{R}^d\dparr
(\subset \T_0)$, the (generically infinite) Lie series over $\R^d$.  cf. \cite{reutenauer2003free,friz2010multidimensional,friz2024rectifiable}.

Equation \eqref{equ:SigSmooth} remains meaningful when $X$ is a sufficiently smooth $\T_0$-valued path, including the important case of $\mathcal{L}$-valued path;  in which case $\bS$ is known as smooth geometric rough path, of recent importance in signature kernel methods \cite{bellingeri2022smooth,lemercier2024highordersolversignature}.

\medskip 
Now, many paths of interest arise as sample paths from stochastic processes, far from continuously differentiable, leave alone smooth. 
A natural class of stochastic processes, where \eqref{equ:SigSmooth} still makes sense - as (Stratonovich) stochastic differential equation - is given by semimartingales.
Let accordingly $\Sec(\R^d)$ denote the class of continuous, $d$-dimensional semimartingales on some filtered probability space $(\Omega, (\mathcal{F}_t)_{t \ge 0}, \mathbb{P})$.
We recall that this is the somewhat decisive class of stochastic processes that allows for a reasonable, stochastic integration theory, notably required by no-arbitrage based continuous-time finance.
Classic texts include \cite{revuz2004continuous,le2016brownian, jacod2003limit,protter2005stochastic}, as a concise introduction we recommend Chapter 1 of \cite{ait2014high}.
Similar to the above one considers, for $X \in \Sec(\R^d)$,
\begin{equation}
     \mathrm d \bS = \bS\,{\circ\mathrm d} X, \qquad \bS_0 = 1  \label{equ:gSig0},
\end{equation}
the solution of which is expressed in terms of iterated (Stratonovich) integrals. As before, \eqref{equ:gSig0} remains meaningful in greater generality, with $X$ replaced by a $\T_0$-valued semimartingale, say $\bX \in \Sec(\T_0)$. We then set $\Sig (\bX)_{0,T} (\omega) := \bS_T (\omega)$, a (random) element in $\T_1$. Since $\T_0$ (resp. $\T_1$) are simple examples of Lie algebras (resp. groups), the linear Stratonovich SDE \eqref{equ:gSig0} then falls into a classical setting \cite{mckean1969stochastic,hakim1986exponentielle}.%

Whenever the (generalized) signature $\Sig (\bX)_{0, T}$ is component-wise integrable, we define the {\em expected signature} and \emph{signature cumulants} by
\[
       \fmu (T) \coloneq  \E (\Sig (\bX)_{0, T})\in\TT_1, \quad \kap (T) \coloneq  \log \fmu (T) \in \tf.
\]

Already in the case 
where $\bX$ is deterministic and sufficiently regular, this leads to an interesting (ordinary differential) equation for $\kap$ with accompanying (Magnus) expansion, well understood as effective computational tool \cite{iserles2005lie,blanes2009magnus}. Lyons \cite{LyonsICM} is an excellent survey
with a variety of applications, ranging from machine learning to numerical algorithms on Wiener space known as cubature \cite{lyons2004cubature}; signature cumulants were named and first studied in their own right in \cite{bonnier2019signature}, providing, in particular, formulas for converting signature moments to signature cumulants and vice versa.

In the special case of $d=1$ and $\bX=(0,X,0,\dots)$ where $X$ is a scalar semimartingale, $\fmu (T)$ and $\kap (T)$ are nothing but the sequence of moments and cumulants of the real valued random variable $X_T-X_0$. When $d > 1$, signature moments / cumulants provide an effective way
to describe the process $X$ on $[0,T]$, see \cite{LJZ11, LyonsICM, chevyrev2016characteristic}. The question arises how to compute them. If one takes $\bX$ as $d$-dimensional Brownian motion, the signature cumulant $\kap(T)$ equals $(T/2) \I_d$, where $\I_d$ is the identity $2$-tensor over $\R^d$. This is known as {\em Fawcett's formula}, \cite{lyons2004cubature,friz2020course}. %
 
Exposing and building on \cite{FHT22}, loosely speaking, our main functional equations (\Cref{thm:main_esig} and \Cref{thm:main_sigcum}) are a vast generalization of Fawcett's formula.
 Projecting these equations to tensor levels, we obtain recursions for signature moments and signature cumulants respectively (\Cref{cor:recursion_mu} and \Cref{cor:recursion_h_form}).
 These recursions represent $\esig^{(n)}$ in terms of $\esig^{(1)}, \dots, \esig^{(n-1)}$ and $\kap^{(n)}$ in terms of $\kap^{(1)}, \dots, \kap^{(n-1)}$ using the underlying dynamic structure of the probability space.
 
 This chapter is organized as follows: Section~\ref{sec:preliminaries} introduces the mathematical preliminaries, including tensor algebra, continuous semimartingales taking values therein, generalized signatures, and their conditional expectations. In Section~\ref{sec:functional}, we derive functional equations for the expected signature and signature cumulants, transitioning from discrete-time processes to continuous semimartingales. Section~\ref{sec:reucsrion} develops the resulting recursive formulas for signature moments and cumulants and demonstrates how they relate to the so-called diamond products of semimartingales. In Section~\ref{sec:multivariate}, we examine the implications of the general functional equations for multivariate moments and cumulants by projecting to the symmetric algebra. In Section~\ref{sec:applications}, we apply the main results to time-inhomogeneous Lévy processes and Brownian rough paths. Finally,  Section~\ref{sec:conclusion} summarizes all main results and corollaries and gives an outlook to potential applications and directions for future research.

\textbf{Acknowledgment.}
PKF and NT acknowledge seed support from DFG CRC/TRR 388 “Rough Analysis, Stochastic Dynamics and Related Fields”, projects A02, A05 (PF) and B01 (NT). PKF is also supported by the DFG Excellence Cluster MATH+ and a MATH+ Distinguished Fellowship.

\section{Preliminaries}\label{sec:preliminaries}
\subsection{The tensor algebra and tensor series}\label{sec:tensor_series}
Denote by $T(\Rd)$ the tensor algebra over $\Rd$, i.e.
\begin{align*}
T(\Rd)\coloneq \bigoplus_{k=0}^\infty (\Rd)^{\otimes k},
\end{align*}
elements of which are {\it finite} sums (tensor polynomials) of the form
\begin{equation} \label{equ:bxform}
\bx = \sum_{k \ge 0} \bx^{(k)} %
= \sum_{w \in \W^d} \bx^w e_w
\end{equation}
with $\bx^{(k)} \in (\Rd)^{\otimes k}, \bx^w \in \R$ and linear basis vectors $e_w \coloneq e_{i_1}\dotsm e_{i_k}\in(\Rd)^{\otimes k}$ where $w$ ranges over all words $w=i_1\dotsm i_k\in\W_d$ over the alphabet $\{1,\dots,d\}$. Note $\bx^{(k)} =  \sum_{|w|=k} \bx^w e_w$ where $|w|$ denotes the length a word $w$. 
The element $e_\emptyset = 1 \in (\Rd)^{\otimes 0} \cong \R$ is the neutral element of the concatenation (tensor) product, which is obtained by linear extension of \(e_we_{w'}=e_{ww'}\) where $ww' \in \W_d$ denotes concatenation of two words.
We thus have, for $\bx,\by \in T(\Rd)$,
\begin{equation}
\label{eq:series.product}
\bx\by = \sum_{k \ge 0} \sum_{\ell =0}^k \bx^{(\ell)} \by^{(k-\ell)}  = \sum_{w \in \W^d} \left( \sum_{w_1w_2 = w} \bx^{w_1}\by^{w_2} \right) e_w \in T(\Rd).
\end{equation}
Note that we reserve the usual product symbol $\otimes$ for another product that will be introduced below. This will also prevent an overflow of product symbols throughout the text.

This extends naturally to {\em infinite} sums, i.e., tensor series, elements of the ``completed'' tensor algebra
\begin{align*}
  \TT \coloneq T\dparl\R^d\dparr\coloneq \prod_{k=0}^\infty (\Rd)^{\otimes k},
\end{align*}
which are written as in (\ref{equ:bxform}), but now as formal infinite sums with identical notation and multiplication rules; the resulting algebra $\TT$ obviously extends $T(\R^d)$.
Denote by $\TT_0$ and $\TT_1$ the subspaces of tensor series starting with $0$ and $1$ respectively;
that is, \(\bx \in\tf\) (resp. \(\TT_1\)) if and only if \(\bx^\emptyset=0\) (resp. \(\bx^\emptyset=1\)).
Restricted to $\tf$ and $\TT_1$ respectively, the exponential and logarithm in $\TT$, defined by the power series,
\begin{align*}
\exp\colon\tf \to \TT_1,& \quad \bx \mapsto \exp(\bx) \coloneq   1 + \sum_{k=1}^\infty \frac{1}{k!}\bx^k, \\
\log\colon\TT_1 \to \tf,& \quad 1 + \bx \mapsto \log( 1 + \bx) \coloneq \sum_{k=1}^\infty \frac{(-1)^{k+1}}{k}\bx^k,
\end{align*}
are globally defined and inverse to each other. 
We will usually abbreviate $e^{\bx} = \exp(\bx)$.
The vector space $\tf$ becomes a {\it Lie algebra} with the commutator bracket \([-,-]\colon\tf\otimes\tf\to\tf\) given by
$$\Lie{\bx}{\by} \coloneq \bx\by-\by\bx, \quad \bx, \by \in \tf.$$
Define the adjoint operator associated to a Lie-algebra element $\by \in \tf$ by
\begin{align*}
\ad{\by}\colon \tf \to \tf, \ \bx \mapsto \Lie{\by}{\bx}.
\end{align*}
The exponential image $\TT_1=\exp(\tf)$ is a Lie group, at least formally so. We refrain from equipping the infinite-dimensional $\TT_1$ with a differentiable structure, not necessary in view of the ``locally finite'' nature of the group law $(\bx,\by) \mapsto \bx \by$, by which we mean that the coefficients of the product \(\bx\by\) can be computed with a finite number of operations as is clear from \cref{eq:series.product}.

For two tensor series \(\bx,\by\) we define their \emph{outer tensor product}
\[
    \bx\otimes \by \coloneqq\sum_{n=0}^{\infty} \left(\sum_{i=0}^{n} \bx^{(i)}\otimes \by^{(n-i)}\right).
\]
We emphasize that here $\otimes$  does \emph{not} denote the (inner tensor) product in $\TT$, for which we did not reserve a symbol,  but it denotes another (outer) tensor product. In particular, this expression is not the same as \cref{eq:series.product}.

Given linear maps \(g,f\colon\TT\to\TT\) we define
\[
    (g\odot f)(\bx\otimes\by) = g(\bx)f(\by)\in\TT
\]
whenever \(\bx,\by\in\TT\), where the multiplication on the right-hand side is performed in \(\TT\) according to \cref{eq:series.product}.
It then extends by linearity to the linear span of outer tensors which we will denote by \(\TT\otimes\TT\).

For \(n\in\N\), the subspace \[ \ideal_n\coloneq\prod_{k=n+1}^\infty(\Rd)^{\otimes k} \] is a two sided
ideal of \(\TT\), meaning that \(\TT\ideal_n+\ideal_n\TT\subseteq\ideal_n\), which turns out to be the right condition under which the quotient space \(\TT/\ideal_n\) has a natural algebra structure.
We can identify \(\TT/\ideal_n\) with %
\[
    \TT^n \coloneq \bigoplus_{k=0}^n(\Rd)^{\otimes k},
\]
equipped with the truncated tensor product
$$
\bx\by = \sum_{k=0}^{n} \sum_{\ell_1 + \ell_2=k} \bx^{(\ell_1)} \by^{(\ell_2)}  = \sum_{w \in \W^d,|w|\le n}
\left( \sum_{w_1w_2 = w} \bx^{w_1}\by^{w_2} \right) e_w \in \TT^n,
$$
which formally is the same as \cref{eq:series.product} after setting \(\bx^{(k)}=0\) for all \(k>n\) as the above definition suggests.
We denote the canonical \emph{projection map} by \(\pi_{(0,n)}\colon\TT\to\TT^n\).
The usual power series in $\TT^{n}$ define $\exp_n\colon
\TT^n_0 \to \TT^n_1$ with inverse $\log_n\colon \TT^n_1 \to \TT^n_0$, and we may again abuse notation and write $\exp$ and $\log$ when no confusion arises. 
As before, $\TT^n_0$ has a natural Lie algebra structure, and $\TT^n_1$ (now finite dimensional) is
a \emph{bona fide} Lie group.

We equip $T(\R^d)$ with the norm
\begin{align*}
    |a|_{T(\R^d)} \coloneq  \max_{k\in\N}|a^{(k)}|_{(\Rd)^{\otimes k}},
\end{align*}
where $|\cdot|_{(\Rd)^{\otimes k}}$ is the euclidean norm on $(\Rd)^{\otimes k}\cong\R^{d^k}$, which makes it a Banach space.
The same norm makes sense in \(\TT^n\), and since the definition is consistent in the sense that $|a|_{\TT^k} = |a|_{\TT^n}$ for any $a \in \TT^{n}$ and $k \ge n$, and $|a|_{\TT^n} = |a|_{(\Rd)^{\otimes n}}$ for any $a \in (\Rd)^{\otimes n}$. We will drop the index whenever it is possible and write simply $|a|$.
For a word $w \in \W_d$ with $|w|>0$ we define the directional derivative for a function $f\colon \TT \to \R$
by \((\partial_w f)(\bx)\coloneq\partial_t f(\bx + t e_w)\big\vert_{t=0},\) whenever the derivative exists.
In particular, for the exponential function we have
\begin{align*}
  \partial_w\exp(\bx) &= G(\ad\bx)(e_w)\exp(\bx)\\
  \partial_{w'}\partial_w\exp(\bx) &= Q(\ad\bx)(e_w\otimes e_{w'})\exp(\bx),
\end{align*}
where the operators \(G(\ad\bx)\colon\TT_0\to\TT_0\) and \(Q(\ad\bx)\colon\TT_0\otimes\TT_0\to\TT_0\) are defined by the power series
\begin{equation}\label{eq:GQ_def}
  \begin{split}
    G(\ad\bx)(\by) &= \frac{e^{\ad\bx}-1}{\ad\bx}(\by) = \sum_{k=0}^{\infty}
    \frac{1}{(k+1)!}(\ad\bx)^k(\by)\\
                    &= \by+\frac12[\bx,\by] + \frac16 [\bx,[\bx,\by]]+\dotsb\\
    Q(\ad\bx)(\by\otimes\by') &=
    2\sum_{n,m=0}^{\infty}\frac{(\ad\bx)^n\odot(\ad\bx)^m}{(n+1)!m!(n+m+2)}(\by\otimes\by')\\
                               &= \by\by' + \frac13[\bx,\by]\by' + \frac13\by[\bx,\by'] + \frac16[\bx,\by][\bx,\by']+\dotsb
  \end{split}
\end{equation}

The first identity is well known and can be traced back to Schur \cite{schur1891theorie}.
With this identities one also obtains the following formal Taylor expansion of the exponential map:
\begin{equation}\label{eq:exp.taylor}
  \exp(\bx)\exp(-\by)-1 = G(\ad\bx)(\bx-\by) + Q(\ad\bx)((\bx-\by)\otimes(\bx-\by))+\dotsb
\end{equation}
where we recall that the symbol \(\otimes\) is meant as an \emph{outer} tensor product.
Naturally, similar formulas hold in the truncated tensor algebra where \(\exp\) is replaced by \(\exp_n\).

\subsection{Semimartingales}\label{sec:tensor_semimartingales}
Let $\D$ be the space of continuous adapted processes $X\colon\Omega \times [0,T) \to \mathbb{R}$, with $T\in(0,\infty]$, defined on some filtered probability space $(\Omega, (\F_t)_{0 \le t \le T}, \mathbb{P})$.
The space of {\em continuous semimartingales} $\Sec$ is given by the processes $X\in\D$ that can uniquely be decomposed as
$$
     X_{t}=X_0+M_{t}+A_{t},
$$
where $M \in \Ma_\loc$ is a continuous local martingale, and $A\in\Fv$ is an adapted process of locally bounded variation, both started at zero.
The (predictable) quadratic variation process of $X$ is denoted by \(\langle X\rangle_t\).
Covariation angle brackets $\CV{X}{Y}$, for another real-valued semimartingale $Y$, are defined by polarization.
For $q \in [1, \infty)$, write $\Lcal^{q} = L^{q}(\Omega, \F, \PM)$, then a Banach space $\HSe^{q} \subset \Sec$ is given by those $X\in\Sec$ with $X_0 = 0$ and
$$
 \| X \|_{\HSe^q} \coloneq \bigg\Vert \QV{M}^{1 / 2}_{T} + \int_0^{T} |\dd A_s | \bigg\Vert_{\Lcal^{q}}< \infty.
$$
For a process $X\in\D$ we define
\begin{align*}
\Vert{X}\Vert_{\Sesup^q} \coloneq \Big\Vert{\sup_{0 \le t \le T}\abs{X_t}}\Big\Vert_{\Lcal^{q}}
\end{align*}
and define the space $\Se^q\subset\Sec$ of semimartingales $X\in\Sec$ such that $\Vert{X}\Vert_{\Sesup^q}< \infty$.
Note that there exists a constant $c_q>0$ depending on $q$ such that (see \cite[Ch. V, Theorem 2]{protter2005stochastic})
\begin{align}\label{eq:Sesup_vs_HSe}
\Vert{X}\Vert_{\Sesup^q} \le c_q  \| X \|_{\HSe^q}.
\end{align}

We view $d$-dimensional semimartingales, $X= \sum_{i=1}^d X^i e_i \in \Sec(\R^d)$, as special cases of tensor series valued semimartingales $\Sec(\T)$ of the form
$$\bX = \sum_{w \in \W_d} \bX^w e_w$$
with each component $\bX^w$ a real-valued continuous semimartingale.
This extends mutatis mutandis to the spaces of $\TT$-valued adapted continuous processes $\D(\TT)$, martingales $\Ma(\TT)$ and processes of finite variation $\Fv(\TT)$.
Note also that we typically deal with $\TT_0$-valued semimartingales $\bX$ which amounts to have only words with length $|w| \ge 1$.
Standard notions such as continuous local martingale $\bX^{c}$ and jump process $\Delta \bX_t = \bX_t - \bX_{t^-}$ are defined componentwise.
To weaken integrability requirements later, we often deal semimartingales taking values in the truncated tensor algebra $\TT^N$ and will omit specifically extending notations when it is straightforward.

{\bf Brackets}: Now let $\bX$ and $\bY$ be $\TT$-valued semimartingales.
We define the (non-commutative) \emph{outer quadratic covariation bracket} of $\bX$ and $\bY$ by
\[
\outerbracket{\bX}{\bY}_t
\coloneq \sum_{w_1, w_2\in\W_d}\langle\bX^{w_1}, \bY^{w_2}\rangle_te_{w_1} \otimes e_{w_2}\in\TT\otimes\TT.
\]
Similarly, define the (non-commutative) {\em inner quadratic covariation bracket} by %
\[
\CV{\bX^{}}{\bY^{}}_t \coloneq \sum_{w\in\W_d%
}\left(\sum_{w_1 w_2 = w}
   \innerbracket{\bX^{w_1}}{\bY^{w_2}}_t \right)e_w \in \TT.
\]
As usual, we may write $\cvbracket{\bX} \equiv \outerbracket{\bX}{\bX}$ and $\QV{\bX} \equiv   \CV{\bX}{\bX}$.

{\bf $\HSe$-spaces}: The definition of $\HSe^{q}$-norm naturally extends to tensor valued martingales.
More precisely, for $\bX^{(n)} \in \Sec((\Rd)^{\otimes n})$ with $n\in\None$ and $q\in[1,\infty)$ we define
\begin{equation*}
\Abss{\bX^{(n)}}_{\HSe^{q}} \coloneq \Abss{\bX^{(n)}}_{\HSe^{q}((\Rd)^{\otimes n})} \coloneq
\Abs{\abs{\QV{\bM}}_T^{1/2} + \abs{\bA}_{1-\mathrm{var};[0;T]}}_{\Lcal^{q}},
\end{equation*}
where $\bX^{(n)} = \bM + \bA$ with $\bM\in\Ma_\loc((\Rd)^{\otimes n})$ and $\bA\in\Fv((\Rd)^{\otimes n})$, with
\begin{equation*}
\abs{\bA}_{1-\var;[0;T]}\coloneq\sup_{0 \le t_1 \le \dotsb \le t_k \le T}
\sum_{t_i}\abs{\bA_{t_{i+1}} - \bA_{t_{i}}} \le \sum_{w\in\W_d, |w|=n} \int_0^{T}\abs{\dd A^{w}_s},
\end{equation*}
with the supremum taken over all partitions of the interval $[0,T]$.
One may readily check that
\begin{equation*}
\Abss{\bX^{(n)}}_{\HSe^{q}} \le \sum_{w\in\W_d, |w|=n} \Abs{\bX^{w}}_{\HSe^{q}}
\end{equation*}
and for $\bX^{(n)}\in\Ma_\loc((\Rd)^{\otimes n})$:
\begin{equation*}
\Abss{\bX^{(n)}}_{\HSe^{q}} = \Abs{\abss{\QVsmall{\bX^{(n)}}}^{1/2}_T}_{\Lcal^{q}}.
\end{equation*}
Further define the following subspace $\HSehom^{q,N} \subset \Sec(\tf^{N})$ of homogeneously $q$-integrable semimartingales
\begin{equation*}
\HSehom^{q,N} \coloneq\left\{ \bX \in \Sec(\tf^{N}) \;\Big\vert\; \bX_0 = 0,\; \homnqN{\bX} < \infty \right\},
\end{equation*}
where for any $\bX \in \Sec(\TT^{N})$ we define
\begin{equation*}
\homnqN{\bX} \coloneq \sum_{n=1}^{N} \big(\Abss{\bX^{(n)}}_{\HSe^{qN/n}}\big)^{1/n}.
\end{equation*}
Note that $\homnqNs{\cdot}$ is sub-additive and positive definite on $\HSehom^{q, N}$ and it is \emph{homogeneous} under dilation in the sense that
\begin{equation*}
\homnqN{\delta_{\lambda}\bX} = \abs{\lambda}\,\homnqN{\bX}, \quad \delta_\lambda \bX \coloneq
(\bX^{(0)}, \lambda \bX^{(1)}, \dotsc, \lambda^{N}\bX^{(N)}), \quad \lambda \in \R.
\end{equation*}
We also introduce the following subspace of $\Sec(\TT)$
\begin{equation*}
\HSe^{\infty-}(\TT) \coloneq \left\{ \bX \in \Sec(\TT):\; \bX^w\in\HSe^q, \;\forall\, 1 \le q < \infty,\;w\in\W_d\right\}.
\end{equation*}
Note that if $\bX\in\Sec(\TT)$ such that $\homns{\bX^{(0,N)}} < \infty$ for all $N\in\None$ then it also holds $\bX\in\HSe^{\infty-}(\TT)$.

{\bf Stochastic integrals}: We are now going to introduce a notation for the stochastic integration with respect to tensor valued semimartingales.
Denote by $\mathcal{L}(\TT; \TT) = \{ f: \TT \to \TT\;|\; f \text{ is linear}\}$ the space of endomorphisms on $\TT$
and let $$\bF: \Omega \times [0,T] \to \mathcal{L}(\TT; \TT), \quad (t, \omega) \mapsto \bF_t(\omega; \cdot)$$ such that
\begin{align}\label{cond:tensor_integrator_one}
&(\bF_t(\bx))_{0 \le t \le T} \in \D(\TT), \quad \text{for all } \bx\in\TT\\\label{cond:tensor_integrator_two}
\text{and}\quad &\bF_t(\omega; \ideal_n) \subset \ideal_n,  \quad \text{for all } n\in\N, \; (\omega, t)\in\Omega\times[0,T],
\end{align}
where $\ideal_n\subset\TT$ was introduced in \Cref{sec:tensor_series}, consisting of series with tensors of level $n+1$ and higher.
In this case, we can define the stochastic Itô integral of $\bF$ with respect to $\bX\in\Sec(\TT)$ by
\begin{align}\label{eq:definition_tensor_sotch_integral}
\int_{(0, \cdot]} \mathbf F_{t}\,\dd\bX_t :=\sum_{w\in\W_d}\;\sum_{v\in\W_d,\,\abs{v} \le \abs{w}}\; \int_{(0, \cdot]}  \mathbf  F_{t}(e_v)^{w}\dd\bX_t^{v} e_w \in \Sec(\TT).
\end{align}
For example, let $\bY, \bZ \in \D(\TT)$ and define $\bF := \bY\,\Id\,\bZ$, i.e. $\bF_t(\bx) = \bY_t \, \bx \, \bZ_t$, the concatenation product from the left and right, for all $\bx\in\TT$.
Then we see that $\bF$ indeed satisfies the conditions \eqref{cond:tensor_integrator_one} and \eqref{cond:tensor_integrator_two} and we have
\begin{equation}\label{eq:def_tensor_ito_integral}
\int_{(0, \cdot]} (\bY_{t}\,\Id\,\bZ_{t})(\dd \bX_t)= \int_{(0, \cdot]} \bY_{t}\dd \bX_t \bZ_{t}= \sum_{w\in\W_d}\left( \sum_{w_1 w_2w_3 = w} \int_{(0, \cdot]}  \mathbf  Z^{w_1}_{t} \bY^{w_3}_{t}\,\mathrm d \bX_t^{w_2} \right)e_w.
\end{equation}
Another important example is given by $\bF = (\ad \bY)^{k} \coloneq \ad\bY \circ \cdots \circ \ad\bY$ ($k$-times) for any $\bY\in\D(\tf)$ and $k\in\N$.
Indeed, we immediately see $\bF$ satisfies the condition \eqref{cond:tensor_integrator_two} and since the iteration of adjoint operations can be expanded in terms of left- and right-multiplication, we also see that $\bF$ satisfies \eqref{cond:tensor_integrator_one}.
More generally, integrals with respect to power series of adjoint operators are well defined.
Indeed, introducing for $\ell \in (\mathbb{N}_{\ge 1})^k$  the notation $\ell = (l_1, \dots, l_k)$, $|\ell|:=k$ and $||\ell||:=l_1 + \dotsb + l_k$ and $(\ad \bx^{(\ell)}) = (\ad \bx^{(l_{1})}  \cdots  \ad \bx^{(l_k)})$ for any  $\bx \in \tf$.
Then for any $(a_k)_{k=0}^{\infty}\subset\R$ and  $\bX\in\Sec(\tf)$ the integral
\begin{align}\label{eq:power_series_integral}
\int_{(0, \cdot]}\left[ \sum_{k=0}^{\infty} a_k (\ad\bY_{t})^{k} \right](\dd\bX_t) = \sum_{n= 1}^{\infty}\sum_{k=0}^{n-1}\; \sum_{\Abs{\ell}=n,\, \abs{\ell} = k+1}a_k \int_{(0,\cdot]}
\ad\bY^{(\ell)}_{t}(\dd\bX^{(l_1)}_t)
\end{align}
is well defined in the sense of \eqref{eq:def_tensor_ito_integral}.
The definition of the integral with integrands of the form $$\bF: \Omega \times [0,T] \to
\mathcal{L}(\TT\otimes\TT; \TT)$$ with respect to processes $\bX \in \Sec(\TT\otimes\TT)$ is completely analogous.

\subsection{Generalized signatures}
\label{sec:gensig}
We now 
give precise meaning to %
\(\mathrm d\bS=\bS\,{\circ\mathrm
d}\bX\), or component-wise, for every word \(w\in\W_d\),
\[
    \mathrm d\bS^w=\sum_{w_1w_2=w}\bS^{w_1}\,{\circ\mathrm d}\bX^{w_2},
\]
where the driving noise \(\bX\) is a \(\tf\)-valued continuous semimartingale, i.e. $\bX \in\Sec(\tf)$.
The Itô integral meaning of this equation, when started at time $s$ from $\bs \in \T_1$, for times $t \ge s$, is given by
\begin{equation}\label{eq:stratonovich_sig}
\bS_t = \bs + \int_{(s,t]} \bS_{u}\,\mathrm d \bX_u + \frac{1}{2}\int_s^{t} \bS_{u}\,\mathrm d\QV{\bX}_u,
\end{equation}
leaving the component-wise version to the reader. We have
\begin{proposition}
    Suppose \(\bX\in\Sec(\tf)\). For every $s \in [0,T]$ and $\bs \in \T_1$, equation \eqref{eq:stratonovich_sig} has a
    unique global solution on \(\TT_1\) starting from $\bS_s=\bs$.
    \label{prop:sigsol}
\end{proposition}
\begin{proof}
   Note that $\bS$ solves \eqref{eq:stratonovich_sig} iff $\bs^{-1} \bS$
   solves the same equation started from $1 \in \T_1$. We may thus take $\bs = 1$ without loss of
   generality. The graded structure of our problem, and more precisely that $\bX = (0,X,\mathbb{X},\dots)$
    in \eqref{eq:stratonovich_sig} has no scalar component, shows that the (necessarily)
    unique solution is given explicitly by iterated integration, as may be seen explicitly when writing out
    $\bS^{(0)} \equiv 1$, $\bS^{(1)}_t = \int_s^t \mathrm d X = X_{s,t} \in \R^d$,
      $$
      \bS^{(2)}_t = \int_{(s,t]} \bS^{(1)}_{u}\,\mathrm d X_u +\mathbb{X}_{t} -\mathbb{X}_{s} + \frac{1}{2} \QV{X}_{s,t} \in (\R^d)^{\otimes 2},
   $$
   and so on. (In particular, we do not need to rely on abstract existence, uniqueness results for SDEs \cite{kurtz1995strtonovich} or Lie group stochastic exponentials \cite{hakim1986exponentielle}.)
   \end{proof}

\begin{definition} \label{def22}
    For \(\bX\in\Sec(\tf)\) and $s\in [0,T]$ we define
    $ \Sig (\bX \vert_{[s,\cdot]}) \equiv \Sig(\bX)_{s,\cdot}$ as the unique solution to \eqref{eq:stratonovich_sig} such
    that \(\Sig(\bX)_{s,s}=1\). We call $\Sig(\bX)_{s,t}$ the generalized signature (short: signature) of $\bX$ on $[s,t]$.
\end{definition}
\begin{remark} In rough path theory, one defines signatures for geometric rough paths. This is consistent with Definition \ref{def22} when $\bX \in \Sec(\mathcal{L})\) (recall $\mathcal{L} = \mathrm{Lie} \dparl\mathbb{R}^d\dparr$), including the special case of a continuous semimartingale with values in $\R^d$.
\end{remark}

The following can be seen as a (generalized) Chen relation.

\begin{lemma} \label{lem:Chen}
Let $\bX\in\Sec(\tf)$ and $0 \le s \le t \le u \le T$. Then the following identity holds with probability one, for all such $s,t,u$,
        \begin{equation}\label{eq:chen_identity}
         \Sig(\bX)_{s,t}\Sig(\bX)_{t,u}=\Sig(\bX)_{s,u}.
         \end{equation}
\end{lemma}
\begin{proof} Call $\Phi_{t \leftarrow s} \bs \coloneq \bS_t$ the solution to \eqref{eq:stratonovich_sig} at time $t \ge s$, started from $\bS_s = \bs$. By uniqueness of the solution flow, we have
$
             \Phi_{u \leftarrow t} \circ \Phi_{t \leftarrow s} = \Phi_{u \leftarrow s} .
$
It now suffices to remark that, thanks to the multiplicative structure of \eqref{eq:stratonovich_sig} we have
$               \Phi_{t \leftarrow s} \bs = \bs  \Sig(\bX)_{s,t}$.
\end{proof}

As usual, when $\bX \in \Sec(\TT^{N}_0)$ for some $N\ge 1$ the development \eqref{eq:stratonovich_sig} is understood within the truncated tensor algebra $\TT^{N}_1$ and therefore $\Sig(\bX)_{s,t}$ is an element in $\TT^N_1$.

\subsection{Expected signatures and signature cumulants}
Throughout this section let $\bX \in \Sec(\tf)$  be defined on a filtered probability space $(\Omega, \F, (\F_t)_{0 \le t \le T}, \PM)$ satisfying the usual conditions and with the property that every martingale has a continuous version.
Recall that $\E_t$ denotes the conditional expectation with respect to the sigma algebra $\F_t$.
When $\E(|\Sig(\bX)^w_{0,t}|)<\infty$ for all $0 \le t \le T$ and all words $w\in\W_d$, then the \emph{(conditional) expected signature}
\begin{equation*}
\esig_t(T) \coloneq \E_t\left(\Sig(\bX)_{t,T}\right) = \sum_{w\in\W_d}\E_t(\Sig(\bX)^w_{t,T})e_w \in \TT_1, \quad 0 \le t \le T,
\end{equation*}
is well defined.
In this case, we can also define the \emph{(conditional) signature cumulant} of $\bX$ by
\begin{align*}
	\kapT{t}\coloneq\log\left(\esig_t(T)\right) \in \tf, \quad 0 \le t \le T.
\end{align*}
An important observation is the following
\begin{lemma}
Given $\E(|\Sig(\bX)^w_{0,t}|)<\infty$ for all $0 \le t \le T$ and words $w\in\W_d$, then $\esig(T) \in \Sec(\TT_1)$ and $\kap(T)\in\Sec(\tf)$.
\end{lemma}
\begin{proof}
It follows from the relation \eqref{eq:chen_identity} that
\begin{align*}
\esig_t(T) = \E_t\left(\Sig(\bX)_{t,T}\right) = \E_t\left(\Sig(\bX)_{0,t}^{-1}\Sig(\bX)_{0,T}\right) = \Sig(\bX)_{0,t}^{-1}\E_t\left(\Sig(\bX)_{0,T}\right).
\end{align*}
Therefore projecting to the tensor components we have
\begin{align*}
  \esig_t(T)^w = \sum_{w_1w_2 = w}(-1)^{|w_1|}\Sig(\bX)^{\overline{w_1}}_{0,t}\E_t\left(\Sig(\bX)^{w_2}_{0,T}\right), \quad 0 \le t \le T, \quad w \in \W_d.
\end{align*}
where \(\overline{w_1}\) denotes word reversal.
Since $(\Sig(\bX)^w_{0, t})_{0 \le t \le T}$ and $(\E_t(\Sig(\bX)^w_{0,T})_{0 \le t \le T}$ are semimartingales (the latter in fact a martingale), it follows from Itô's product rule that $\esig^w(T)$ is also a semimartingale for all words $w\in\W_d$, hence $\esig(T)\in\Sec(\TT_1)$.
Further, recall that $\kap(T) = \log(\esig(T))$ and therefore it follows from the definition of the logarithm on $\TT_1$ that each component $\kap(T)^w$ with $w\in\W_d$ is a polynomial of $(\esig(T)^{v})_{v\in\W_d, |v|\le|w|}$. Hence it follows again by Itô's product rule that $\kap(T)\in\Sec(\tf)$.
\end{proof}
It is of strong interest to have a more explicit necessary condition for the existence of the expected signature.
The following theorem below, the proof of which can be found in \cite[Section~7.2]{FHT22}, yields such a criterion.
\begin{theorem}\label{thm:bdg_signature}
Let $q\in[1, \infty)$ and $N\in\None$, then there exist two constants $c,C>0$ depending only on $d$, $N$ and $q$, such that for all $\bX = \bX^{(0,N)} \in \HSehom^{q,N}$
\begin{equation*}
  c\homnqNs{\bX} \le \homnqNs{\Sig(\bX)_{0,\cdot}} \le C\homnqNs{\bX}.
\end{equation*}
In particular, if $\bX\in\HSe^{\infty-}(\tf)$ then $\Sig(\bX)_{0,\cdot}\in \HSe^{\infty-}(\TT_1)$ and the expected signature exists.
\end{theorem}
\begin{remark}
Let $\bX = (0, M, 0, \dotsc, 0)$ where $M \in \Ma(\Rd)$ is a martingale, then $$\homnqNs{\bX} = \Abss{M}_{\HSe^{qN}} = \Abss{\abs{\QV{M}_T}^{1/2}}_{\Lcal^{qN}},$$
and we see that the above estimate implies that
\begin{align*}
\max_{n=1, \dotsc, N} \Abss{\Sig(\bX)^{(n)}_{0, \cdot}}_{\Se^{qN/n}}^{1/n} \le C \Abss{M}_{\HSe^{qN}}.
\end{align*}
This estimate is already known and follows from the Burkholder-Davis-Gundy inequality for enhanced martingales, which was first proved in \cite{friz2006burkholder}.
\end{remark}

\section{Functional equations for the expected signature}\label{sec:functional}

In this section we present central functional equations for the expected signature and signature cumulants of semimartingales.
For better accessibility we will first present the formulas for discrete processes, which then will inform the results for purely continuous processes via a limiting procedure. 
We will however prove the results in the continuous case directly by a stochastic calculus approach.%

\subsection{Discrete processes}\label{sec:cadlag_case}

Let $(\Omega, \G,\PM)$ be a probability space with filtration $(\G_j)_{j = 0, \dots, J}$.
Further, let $(\bX_j)$ be a $\tf$-valued adapted process.
We can define a discrete version of the signature by  multiplying exponential increments $$\bS_0 = 1, \qquad \bS_{j} = \bS_{j-1} e^{\bX_j - \bX_{j-1}}, \quad \text{ for } j =1, \dots, J.$$
This is precisely the signature of the paths obtained from linearly interpolating the points $\bX_0, \bX_1, \dots, \bX_J$ in $\tf$.
Assume for simplicity that $\bX$ has moments of all orders, i.e.,  $\E[|\bX^{(n)}_j|^p] < \infty$ for all $n, p\in\mathbb{N}_{\ge 1}$,  for otherwise we need to truncate tensor levels.
We then correspondingly define the conditional expected signature of $(\bX_j)$ by
\begin{align*}
\esig_j :=& \E\big[\bS_j^{-1}\bS_J \;\vert \G_j\big] , \qquad j = 0, \dots, J.
\end{align*}
A simple computation yields for  $j = J-1, \dots, 0$:
\begin{align*}
\esig_j =&~ \E\big[e^{\bX_{j+1} - \bX_{j}} \cdots  e^{\bX_{J} - \bX_{J-1}} \;\vert \G_j\big] \\
=&~ \E\big[e^{\bX_{j+1} - \bX_{j}} \E\big[  e^{\bX_{j+2} - \bX_{j+1}} \cdots  e^{\bX_{J} - \bX_{J-1}} \;\vert \G_{j+1}\big]\;\big\vert\G_j\big] \\
=&~ \E\big[e^{\bX_{j+1} - \bX_{j}} \esig_{j+1} \;\vert \G_j\big].
\end{align*}
These identities permit to calculate $\esig$ in a backwards induction.
What is not revealed, however, is that to compute $\esig^{(n)}$ only strictly lower tensor levels $\esig^{(1)}, \dots \esig^{(n-1)}$ are needed.
To see that this is really the case, we first bring the identity into the following difference form
\begin{align}\label{eq:discrete_diff_identity}
0 = \E\big[(e^{\bX_{j+1} - \bX_{j}} - 1)\esig_{j+1} + (\esig_{j+1} - \esig_j) \;\big\vert \G_j\big], %
\end{align}
Summing over $\{j, j+1, \dots, J-1\}$ and conditioning to $\G_j$ we obtain the following
\begin{theorem}\label{thm:discrete} Given that the adapted $\tf$-valued process $(\bX_j)$ satisfies the moment condition $\E[|\bX^{(n)}_j|^p] < \infty$ for all $n, p\in\mathbb{N}_{\ge 1}$, its conditional expected signature $(\esig_j)$ is uniquely characterized by the equation
\begin{align*}
    \esig_j = 1 + \E\bigg[\sum_{i=j}^{J-1}(e^{\bX_{i+1} - \bX_{i}} - 1)\esig_{i+1} \;\bigg\vert \G_j\bigg],\qquad  j=0, \dots, J.
\end{align*}
In particular, projecting to tensor levels yields the following recursion over $n \in \{1, 2, \dots\}$ and $j= \{J, J-1, \dots, 1\}$
\begin{align*}
\esig_j^{(n)} =\sum_{k=0}^{n-1} \sum_{i=j}^{J-1}  \sum_{\Vert\ell\Vert = n-k} \frac{1}{|\ell|!}\E\Big[(\bX_{i+1} - \bX_{i})^{(\ell)}\esig_{i+1}^{(k)} \;\Big\vert \G_j\Big],
\end{align*}
where for $\ell \in \mathbb{N}_{\ge 1}^k$ and $\bx \in \tf$ we set $\bx^{(\ell)}= \bx^{(l_1)} \cdots \bx^{(l_k)}$.
\end{theorem}

\begin{example}(\textit{Markov Chains; see also }\cite[Section~5]{bonnier2023adapted}.)\label{expl:markov}
Consider a $d$-dimensional Markov chain $(X_j)$ and set $(\bX^{(1)}_j) = (X_j)$ and $\bX^{(n)}_j \equiv0$ for $n \neq 1$.
Denote by $p^{(i)}_j(x, \dd{y})$ the step-$i$ transition probability kernel of $X$ starting in $x\in \Rd$ at time $j$ with $p^{(0)}_j(x, \dd{y}) = \delta_{x}(\dd{y})$.
We see inductively that almost surely $\esig_j^{(n)} = f^n_j(X_j)$ where the functions $f^{n}_j : \Rd \to (\Rd)^{\otimes n}$ satisfy the recursive scheme
\begin{align*}
f^n_j(x) =\sum_{k=1}^{n-1} \sum_{i=j}^{J-1}   \frac{1}{n!} \int_{\Rd}\int_{\Rd} (z - y)^{\otimes(n-k)}f_{i+1}^{k}(z)p_i^{(1)}(x, \dd{y}) p_j^{(i-j)}(y, \dd{z})
\end{align*}
for $n = 1, 2\dots$  and $j = J-1, \dots, 0$ with and $f_j^{0} = 1$, $f_J^n = 0$ for $n \ge 1$.
\end{example}

\begin{example} (\textit{Random walks in the group}\cite{breuillard2009from, chevyrev2018random}.)
Let $(g_j)_{j=1, \dots, J}$ be an IID sequence with values in the Lie algebra $\tf^N$ and
exponential image $G_j = e^{g_j} \in \TT_1$.
Consider the random walk \(\bX_j=\sum_{k=1}^jg_k\) and its signature
\[
    \bS_j = \prod_{k=1}^j G_k,\quad j=1,\dotsc,J.
\]
Then either directly, or following the computations above, we obtain the conditional expected signature
$$
    \esig_j = \E\big(\bS_j^{-1}\bS_J\big|\G_j \big) = M^{ J-j},\quad j=0,\dotsc,J
$$
assuming that $M= \E(G_1)$ is (componentwise) well defined.

For instance, the case of a planar lattice random walk is covered by $N=1,d=2$, and equal weighted point masses at $\pm e_1 \pm e_2$, that is, \(\bX\) is a simple random walk in \(\R^2\).
A direct computation shows that
\begin{align*}
    M ~=~ \frac14\left( e^{e_1} + e^{-e_1} + e^{e_2}+ e^{-e_2} \right) ~=~ \frac12\left( \cosh(e_1) + \cosh(e_2) \right),
\end{align*}
where \(\cosh\colon \TT\to\TT\) is defined by its power series (see \Cref{sec:tensor_semimartingales}).
In particular, for any tensor series \(\bx\in\TT_0\),
\[
    \cosh(\bx) = 1 + \frac12\bx^2 +\dotsb
\]
For \(\lambda>0\), let \(\delta_\lambda:\TT\to\TT\) denote the dilation operator.
Diffusive rescaling of moments yields
\[
    \delta_{J^{-1/2}}\esig_0 = (\delta_{J^{-1/2}}M)^{J} = \left( \mathbf{1}+\frac{1}{2J}(e_1^2 + e_2^2) + \operatorname{o}(J^{-1})\right)^{J}.
\]
Upon taking the limit as \(J\to\infty\) we recover Fawcett's formula for planar Brownian motion.

For an example with $N=2$ consider equal weighted point masses at $\{\pm [e_i,e_j] \;:\; 1\le i < j \le d\}$.
A similar argument leads to a Fawcett-type formula
\[
    \lim_{J\to\infty}\delta_{J^{-1/4}}\esig_0(J) = e^{\frac{2}{d(d-1)}\sum_{i<j} ([e_i,e_j])^{2}}.
\]
for the level-$4$ Brownian rough paths that recently arose in \cite{hairer2024renormalization} as limits of fractional Brownian motion with $H\to\frac{1}{4}$ (see also \Cref{sec:brownian_rough_paths}).
\end{example}

We close the discussion on discrete processes by arguing how to extend these results to the continuous case.
Expressing the functional equation from Theorem~\ref{thm:discrete} in the alternative form
\[ \esig_j = 1 +\mathbb{E} \bigg[ \sum_{i = j}^{J - 1}
   (e^{\bX_{i + 1} - \bX_i} - 1) (\esig_i +
   (\esig_{i + 1} - \esig_i)) \;\bigg\vert \G_j \bigg],
\]
we see that in the diffusive limit, setting $\delta \bX_i = \bX_{i + 1} -
\bX_i$, we can expand
\[ \sum_i (e^{\bX_{i + 1} - \bX_i} - 1) \esig_i
   \approx \sum_i (\delta \bX_i + \frac{1}{2} (\delta
   \bX_i)^{2} + \cdots) \esig_i, \]
and similarly, with $\delta \esig_i = \esig_{i + 1} -
\esig_i$,
\[ \sum_i (e^{\bX_{i + 1} - \bX_i} - 1) (\esig_{i +
   1} - \esig_i) \approx \sum_i (\delta \bX_i + \cdots)
   \delta \esig_i, \]
to anticipate the functional relation of Theorem \ref{thm:main_esig}, derived by stochastic calculus.

For the conditional signature cumulants $(\kap_j) = (\log\esig_j)$ we obtain a first identity by multiplying \eqref{eq:discrete_diff_identity} with $e^{-\kap_j}$ from the right:
$$
0 = \E\big[(e^{\bX_{j+1} - \bX_{j}} - 1)e^{\kap_{j+1}}e^{-\kap_j} + (e^{\kap_{j+1}}e^{-\kap_j} - 1) \;\big\vert \G_j\big].
$$
Even in the commutative (one-dimensional) case the above identity gives non-trivial relations for conditional cumulants (see \Cref{sec:diamonds_discrete}).
In contrast to the continuous case, transforming the above formula entirely to logarithmic coordinates under the use of the Baker-Campbell-Hausdorff (BCH) formula is in general not possible, since after all the logarithm is non-linear and cannot be exchanged with the conditional expectation.
Nevertheless, a martingale transform yields the identity
$$
0 = \E\bigg[ \sum_{i=j}^{J-1}H(\ad{\kap_i})\Big((e^{\bX_{i+1} - \bX_{i}} - 1)e^{\kap_{i+1}}e^{-\kap_i} + (e^{\kap_{i+1}}e^{-\kap_i} - 1)\Big) \;\bigg\vert \G_j\bigg],
$$
where following \cref{eq:exp.taylor} a second order expansion in the terms $\delta \bX$ and $\delta \kap_i = \kap_{i+1} - \kap_i$ yields
\begin{align*}
e^{\kap_{i+1}}e^{-\kap_i} - 1 = G(\ad{\kap_i})(\delta \kap_i) + Q(\ad{\kap_i})(\delta \kap_i \otimes \delta \kap_i) + \dots
\end{align*}
and
\begin{align*}
    (e^{\bX_{i+1} - \bX_{i}} - 1)e^{\kap_{i+1}}e^{-\kap_i} = \Big(\delta\bX_i + \frac12(\delta\bX_i)^2 + \dots\Big)\Big(1 + G(\ad{\kap_i})(\delta \kap_i) + \dots\Big),
\end{align*}
allows us to anticipate that the diffusive limit leads to the functional equation given in Theorem~\ref{thm:main_sigcum} below.

\subsection{The continuous case}\label{sec:continuous_case}

Throughout this section we will assume that $(\Omega, \F, (\F_t), \mathbb{P})$ is a filtered probability space satisfying the usual conditions and with the property that every martingale has a continuous version. This is for instance the case if the filtration is the natural and completed filtration of a Brownian motion.

To streamline the derivation, we will start from first principles and introduce the key concepts from Itô-calculus in the non-commutative setting when first needed, not focusing on integrability considerations which are fully resolved in \cite{FHT22}.
For $\bX \in\HSe^{\infty-}(\tf)$ the conditional expected signature and the signature cumulants 
$$\esig (T) = (\E_t(\Sig(\bX)_{t,T}))_{0 \le t \le T} \in \Sec(\TT_1), \qquad \kap (T) = (\log \esig_t(T) )_{0 \le t \le T} \in \Sec(\tf),$$
are well defined.
The derivation starts by defining the martingale $$\bM := (\E_t[\Sig(\bX)_{0,T}])_{0\le t \le T} \in \Ma(\TT_1),$$ which is continuous under the given assumptions on the filtration.
It factors into
\begin{align*}
\bM_t
~=~ \Sig(\bX)_{0,t}\E_t[\Sig(\bX)_{t,T}] 
~=~ \Sig(\bX)_{0,t} \esig_t(T).
\end{align*}
One is now inclined to apply Itô's rule to the product on the right-hand side as this should reveal non-trivial cancellations implied by the martingality of $\bM$.
For ease of notation we denote $\bS_t := \Sig(\bX)_{0,t}$ and $\esig_t := \esig_t(T)$.
Then the definitions put forward in Section~\ref{sec:tensor_semimartingales} directly yield the product rule
\begin{align}\label{eq:derivation_product}
   \bM_t =&~ 1 + \int_0^t\dd\bS_u\,\esig_u + \int_0^t \bS_u\,\dd \esig_u + \innerbracketsmall{\bS}{\esig}_{t} \nonumber\\
   =&~
   1+ \int_0^t \bS_u \left( \Big(\dd{\bX}_u + \frac{1}{2}\dd{\QV{\bX}}_u\Big) \esig_u +  \dd\esig_u + \dd\innerbracketsmall{\bX}{\esig}_{u} \right)
\end{align}
where in the second line we have used the Itô-integral form of $\bS$ in \eqref{eq:stratonovich_sig}.
The martingality of $\bM$ and invertability of the tensor $\bS_t$ for all times $t \in[0,T]$ imply that the term in the bracket is the differential of a local martingale.
By Theorem~\ref{thm:bdg_signature} we have $\bS, \esig \in \HSe^{\infty-}(\tf)$ and it is therefore not difficult to conclude that this local martingale is also a true martingale.
Taking conditional expectations of its integral then yields the following functional equation $\esig$:

\begin{theorem}\label{thm:main_esig}
The conditional expected signature $\esig = \esig(T)$ of $\bX \in\HSe^{\infty-}(\tf)$ is the unique solution (up to indistinguishably) of the following functional equation
\begin{equation}\label{eq:esig_master}
\begin{multlined}
\esig_t  = 1 + \E_t\bigg\{ \int_t^T \Big(\dd{\bX}_u + \frac{1}{2}\dd{\QV{\bX}}_u\Big)\, \esig_u +  \innerbracketsmall{\bX}{\esig}_{t,T}
\bigg\},\qquad 0 \le t \le T.
\end{multlined}
\end{equation}
Furthermore, if $\bX\in\HSehom^{1,N}$ for some $N\in\None$, then the identity
\eqref{eq:esig_master} still holds true for the
truncated expected signature $\esig \coloneq (\E_t(\Sig(\bX^{(0,N)})_{t,T}))_{0\le t \le T}$.
\end{theorem}
It is crucial to understand that the uniqueness part of the above statement follows by projecting the above equation to a tensor level, say $n\ge 1$,  and then noting that the above right-hand side only depends on tensor levels $\esig^{(k)}$ for $k \le n-1$.
In other words, equation \eqref{eq:esig_master} directly leads to a recursive scheme for calculating $\esig$ which is explicitly spelled out in Corollary~\ref{cor:recursion_mu} of Section~\ref{sec:reucsrion}.

Moving forward to the signature cumulants $\kap := \kap_\cdot(T)$ one needs to further resolve the differential of $\esig = e^{\kap}$ in \eqref{eq:derivation_product}.
The key is following Itô-formula for the tensor exponential map.
\begin{lemma}\label{lem:ito_exp_rule} 
For $\bZ \in\Sec(\tf)$ it holds
\begin{equation*}
  e^{\bZ_t} - e^{\bZ_0} = \begin{multlined}[t]\int_0^t G(\ad{\bZ_{u}})(\dd\bZ_u)e^{\bZ_u} 
  + \int_0^{t} Q(\ad{\bZ_{u}})(\dd\outerbracket{\bZ}{\bZ}_u)e^{\bZ_u},
\end{multlined}
\end{equation*}
for all $0 \le t \le T$.
\end{lemma}
A version of the above lemma in a matrix setting previously appeared in \cite{kamm2020stochastic}; the full proof can be found in \cite[Lemma~7.8]{FHT22}.
Loosely speaking, the form of the right-hand side follows immediately from the order-2 Taylor expansion for the tensor exponential given in \cref{eq:exp.taylor}.

Applying \Cref{lem:ito_exp_rule} to \(\kap\) in \eqref{eq:derivation_product} and collecting terms, we arrive at
\begin{align*}
\bM_t = 1 + \int_0^t \bS_u \dd \bL_u  e^{\kap_u(T)},
\end{align*}
where $\bL \in \Sec(\tf)$ is given by
\begin{align*}
\begin{split}
\bL_t =~& \bX_{t} + \frac{1}{2}\QV{\bX}_{t} +
    \int_0^{t} G(\ad{\kap_{u-}})(\dd\kap_u) + \frac{1}{2}\int_t^{T}Q(\ad{\kap_{u-}})(\dd\outerbracket{\kap}{\kap}_u) \\
&+\int_t^{T}(\Id\odot G(\ad{\kap_{u-}}))(\dd\outerbracket{\bX}{\kap}_u).
\end{split}
  \end{align*}
Again the martingality of $\bM$ and invertability of $\bS_t, e^{-\kap_t} \in \Sec(\TT_1)$ imply that $\bL$ is a local martingale.
The main work done in the proof of \cite[Theorem~4.1]{FHT22} is to show that under the given assumptions on $\bX$ it also holds that $\bL$ is also a true martingale.
Taking expectations then yields a first functional equation for $\kap$:
$$\E_t[\bL_{t,T}] = 0.$$
Similarly to the functional equation \eqref{eq:esig_master} for $\esig$, the above equation leads to a recursive scheme over tensor levels which uniquely determines $\kap$.
However,  there is still a degree of implicitness that can be removed.

To this end note that the operator $G(\ad{\bx})$ has the following inverse\footnote{This follows from the fact that for linear operators, power series calculus is multiplicative in the sense that the composition \(G(L)H(L)\) equals \(GH(L)\) where \(GH(z)\) is the Cauchy product of \(G\) and \(H\). Moreover, \(G(z)\) and \(H(z)\) are multiplicative inverses as power series so that \(GH(z) = 1\).}
\begin{equation}\label{eq:GH_def}
    H(\ad{\bx}) = \frac{\ad\bx}{\exp(\ad\bx)-1} \coloneq \sum_{k=0}^{\infty}\frac{B_k}{k!}(\ad{\bx})^{k},
\end{equation}
with Bernoulli numbers $(B_k)_{k\ge 0} = (1, -\frac{1}{2}, \frac{1}{6}\dotsc)$.
Integrating $H(\ad{\kap})$ against $\dd\bL$ and verifying that the resulting process is in $\Ma(\tf)$ we arrive at a second functional equation for $\kap$:

\begin{theorem}\label{thm:main_sigcum}
The signature cumulant $\kap$ of $\bX \in\HSe^{\infty-}(\tf)$ is the unique solution (up to indistinguishably) of the following functional equation
\begin{align}\label{eq:master_sigcum}
\begin{multlined}
\kap_t = \E_t\bigg\{ \int_t^T H(\ad{\kap_{u}})(\mathrm d\bX_{u})
+ \frac{1}{2}\int_t^{T}H(\ad{\kap_{u}})(\mathrm d\QV{\bX}_{u})\\
 + \frac{1}{2} \int_t^{T}H(\ad{\kap_{u}}) \, Q(\ad{\kap_{u}})(\mathrm d\outerbracket{\kap}{\kap}_u)\\
+\int_t^{T}H(\ad{\kap_{u}})\, (\Id\odot G(\ad{\kap_{u}}))(\mathrm d\outerbracket{\bX}{\kap}_u)
\bigg\},%
\end{multlined}
\end{align}
for all $0 \le t \le T$.
Furthermore, if $\bX\in\HSehom^{1,N}$ for some $N\in\None$, then the identity \eqref{eq:master_sigcum} still holds true for the
truncated signature cumulant $\kap \coloneq (\log\E_t(\Sig(\bX^{(0,N)})_{t,T}))_{0\le t \le T}$.
\end{theorem}

\section{Recursive formulas and diamond products}\label{sec:reucsrion}

Theorem \ref{thm:main_esig} and \ref{thm:main_sigcum} allow for an iterative computation of the expected signature and signature cumulants, by projecting the functional equations to tensor levels. 
In the following subsections we will explicitly spell out these recursive schemes and then rephrase them in terms of so called {\em diamond products} as they were introduced for scalar semimartingales in \cite{alos2018exponentiation}. 

While all recursions can be stated for the general \cadlag case (\cite[Section~4.2]{FHT22}), we only present the continuous case for simiplicity. Therefore we shall presume the assumptions on the filtration from the continuous setting in Section~\ref{sec:continuous_case} throughout this section. 

\subsection{Recursive formula for signature moments}
This recursive scheme trivially starts with $\fmu^{(0)} \equiv 0$ and the first level of signature moments
\begin{align*}
     \fmu^{(1)}_t = \E_t\left(\bX^{(1)}_{t, T}\right).
\end{align*}
and continuous with a first recursion in the second level
\begin{align}\label{eq:level_two_esig}
  \esig_t^{(2)} &= \mathbb{E}_t\bigg[ \bX_{t,T}^{(2)} + \frac{1}{2}\QV{\bX^{(1)}}_{t,T} + \int_t^T\dd \bX^{(1)}_u \,\esig^{(1)}_{u}+ \CV{\bX^{(1)}}{\esig^{(1)}}_{t,T}\bigg].
\end{align}
Projecting the functional equation \eqref{eq:esig_master} for $\esig$ to higher tensor levels immediately gives the following
\begin{corollary}\label{cor:recursion_mu}Let $\bX\in\HSehom^{1,N}$ for some $N\in \mathbb{N}_{\ge 1}$. Then for $n = 1, \dots, N$ we have
\begin{multline}\label{eq:recursion_esig}
  \esig_t^{(n)} = \sum_{k=0}^{n-1}\mathbb{E}_t\bigg\{ \int_t^T\mathrm{d}\bX^{(n-k)}_u\,\esig^{(k)}_u+\frac12\sum_{i=1}^{n-k}\int_t^T\mathrm{d}\QV{\bX^{(i)},\bX^{(n-k-i)}}_u\,\esig_u^{(k)} \\+  \QV{\bX^{(n-k)},\esig^{(k)}}_{t,T}\bigg\}.
\end{multline}
\end{corollary}
\begin{example}\label{exmpl:esig_recursion_martingale}
Consider the special case with vanishing higher order components, $\bX^{(i)} \equiv 0$, for $i \ne 1$,
and $\bX = \bX^{(1)} \equiv M$, a $d$-dimensional continuous square-integrable
martingale. %
We then have $\fmu^{(1)} \equiv 0$ and it directly follows from Stratonovich-Itô correction that
$$
 \esig^{(2)}_t =  \E_t \int_t^T (M_u - M_t) \circ \dd M_u = \frac12 \E_t \QV{M}_{t,T} =  \frac12 \E_t \QV{\bX^{(1)}}_{t,T}.
$$
which is indeed a (very) special case of the general expression for $\esig^{(2)}$.

Assuming that $M \in \HSe^{N}$
the recursion proceeds for levels $n\in \{2,  3,\dots, N\}$ and simplifies due to the martingality to
\begin{align*}
  \esig_t^{(n)} = \mathbb{E}_t\bigg\{\frac12\int_t^T\mathrm{d}\CV{M}{M}_u\, \esig_u^{(n-2)} + \CV{M}{\esig^{(n-1)}}_{t,T}\bigg\}.
\end{align*}
In case $M$ is a Gaussian martingale of the form $M_t = \int_0^t \sigma(s). \dd{B_s} \in \HSe^{\infty-}$ for an $m$-dimensional Brownian motion $B$ and $\sigma \in L^2([0,T]; \R^{d\times m})$,  an induction immediately shows that $\esig$ is deterministic.
Whence brackets with $\esig$ vanish and we obtain $\esig^{(n)} \equiv 0$ for odd $n$ and  
\begin{align}\label{eq:esig_gaussian martingale}
  \esig_t^{(n)} = \frac12\int_t^T\sigma.\sigma^T(u) \;\esig_u^{(n-2)}\; \dd{u}, \qquad\text{ for even }n.
\end{align}
When $M$ is a standard Brownian motion, i.e. , $\sigma \equiv \I_d \in \R^{d\times d}$ is the identity matrix,  we now readily verify Fawcett's formula \cite{fawcett2002problems,friz2020course}
\begin{align*}
\esig_t = \exp\Big(\frac{T-t}{2} \I_d\Big), \qquad t\in[0,T].
\end{align*}
\end{example}
\begin{remark}
Fawcett's formula is reminiscent of the case of real-valued Gaussian vectors, which are in fact characterized by having non-vanishing cumulants only at degree 2, and one might wonder whether that is also the case for expected signatures as they play a role analogue to characteristic functions for paths.
In that case, the formula is valid \emph{only} for Brownian motion.
Other Gaussian processes may have arbitrarily complicated cumulants as can already be seen from \Cref{exmpl:esig_recursion_martingale} above. 
For an even more elementary example, let \(d>1\) and consider a centered Gaussian vector \(Z=(Z^1,\dotsc,Z^d)\in\R^d\) with covariance matrix \(\Sigma\), and \(X_t\coloneqq tZ\).
Clearly, the law of \(X_t\) is Gaussian for every fixed \(t>0\).
Moreover,
\[
  \Sig(X)_{t,T}=\exp\left( (T-t)Z \right)
\]
so that the first few terms of the expected signature are \(\esig^{(1)}_t(T)=\esig^{(3)}_t(T)=0\) (in fact all odd degrees vanish), and
\begin{align*}
  \esig^{(2)}_t(T) &= \frac{(T-t)^2}{2}\sum_{i_1,i_2=1}^d\Sigma^{i_1i_2}e_{i_1i_2}\\
  \esig^{(4)}_t(T) &= \frac{(T-t)^4}{4!}\sum_{i_1,i_2,i_3,i_4=1}^d(\Sigma^{i_1i_2}\Sigma^{i_3i_4}+\Sigma^{i_1i_3}\Sigma^{i_2i_4}+\Sigma^{i_1i_4}\Sigma^{i_2i_3})e_{i_1i_2i_3i_4}
\end{align*}
where in the last term we have used Isserli's theorem to express mixed moments as products of covariances.
Using the power series expansion of \(\log\) we see that the first four signature cumulants are \(\kap^{(1)}_t(T)=\kap^{(3)}_t(T)=0\), and
\begin{align*}
  \kap^{(2)}_t(T) &= \esig^{(2)}_t(T) \\
  \kap^{(4)}_t(T) &= \esig^{(4)}_t(T) - \frac12\esig^{(2)}_t(T)\esig^{(2)}_t(T) \\
  &= \frac{(T-t)^4}{4!}\sum_{i_1,i_2,i_3,i_4=1}^d(\Sigma^{i_1i_3}\Sigma^{i_2i_4}+\Sigma^{i_1i_4}\Sigma^{i_2i_3}-2\Sigma^{i_1i_2}\Sigma^{i_3i_4})e_{i_1i_2i_3i_4}.
\end{align*}
In particular, the fourth cumulant does not vanish.
A general diagramatic expansion, analog to Isserlis' theorem, for the expected signature of a Gaussian processes in terms of its correlation function has been recently obtained by Cass and Ferrucci \cite{cass2024wiener}.
\end{remark}

\subsection{Recursive formula for signature cumulants}
The first level of signature cumulants are identical to the signature moments $$\kap^{(1)}_t = \fmu^{(1)}_t = \E_t\left(\bX^{(1)}_{t, T}\right).$$
The formula we obtain from projecting \eqref{eq:master_sigcum} to the second tensor level reads
\begin{align*}
    \kap^{(2)}_t & = \E_t \bigg[\bX^{(2)}_{t,T}
    + \frac{1}{2}\QV{\bX^{(1)}}_{t,T} +\frac12 \int_t^T\Lie{\kap^{(1)}_{u}}{\dd \bX^{(1)}_u} + \frac{1}{2}\QV{\kap^{(1)}}_{t,T} + \CV{\bX^{(1)}}{\kap^{(1)}}_{t,T}\bigg].
\end{align*}
From the very definition of the logarithm we must have $\kap^{(2)} = \esig^{(2)} - \frac12 \fmu^{(1)} \esig^{(1)},$
which is in line with the above expression for $\kap^{(2)}$ and the previous formula \eqref{eq:level_two_esig} for $\esig^{(2)}$, but already requires a little work to verify.
Indeed, using the martingality of $\bX^{(1)}+\kap^{(1)} = \E_\cdot[\bX^{(1)}_T]$ we resolve both expressions using the Itô-product rule in the following calculation
\begin{align*}
    \E_t \bigg[\frac12 \int_t^T\ \dd \bX^{(1)}_u\,\kap^{(1)}_{u}\bigg] 
    &= \E_t \bigg[-\frac12 \int_t^T\ \dd \kap^{(1)}_u\,\kap^{(1)}_{u}\bigg] \\
    &=\frac12\E_t \bigg[-\kap^{(1)}_{t} \kap^{(1)}_{t} - \frac12\int_t^T\ \kap^{(1)}_{u}\dd \bX^{(1)}_u + \frac{1}{2}\QV{\kap^{(1)}}_{t,T}\bigg].
\end{align*}

The general recursion for $\kap$ is combinatorially more involved.
To present it as concise as possible, we recall the following notation from \Cref{sec:tensor_semimartingales}: 
For $\ell \in (\mathbb{N}_{\ge 1})^k$ we write $\ell = (l_1, \dots, l_k)$, $|\ell|:=k$ and $||\ell||:=l_1 + \dotsb + l_k$.
Furthermore, for $0 \le i, j \le k$ we define $\ell_{i:j} = (l_{i+1}, \dots, l_j)$ if $i < j$ and $\ell_{i:j} = ()$ otherwise.
Moreover $(\ad \bx^{(\ell)}) = (\ad \bx^{(l_{1})}  \cdots  \ad \bx^{(l_k)})$ for any  $\bx \in \tf$ and $(\ad \bx^{()}) = \Id$.
\begin{corollary}\label{cor:recursion_h_form} Let $\bX\in\HSehom^{1,N}$ for some $N\in \mathbb{N}_{\ge 1}$, then we have
\begin{multline}\label{eq:recursion_h_form}
\kap^{(n)}_t = \E_t\left(\bX^{(n)}_{t,T}\right) +
\sum_{ |\ell|\ge 2,\; ||\ell||=n} \E_t\bigg( \mathrm{HMag}^{1}(\bX, \kap; \ell)_{t,T} + \frac{1}{2}\mathrm{HMag}^{2}(\bX, \kap; \ell)_{t,T} \\+  \mathrm{HQua}(\kap; \ell)_{t,T} + \mathrm{HCov}(\bX, \kap; \ell)_{t,T} \bigg)
\end{multline}
where the summation is over $\ell \in \mathbb{N}_{\ge1}^k$, $k\in\mathbb{N}_{\ge1}$ and
\begin{align*}
\mathrm{HMag}^1(\bX, \kap; \ell)_{t,T} &= \frac{B_{k-1}}{(k-1)!} \int_t^T\ad{\kap^{(\ell_{1:k})}_{u}} \left(\dd\bX^{(l_1)}_u\right) \\
\mathrm{HMag}^2(\bX, \kap; \ell)_{t,T} &= \frac{B_{k-2}}{(k-2)!} \int_t^{T} \ad{\kap^{(\ell_{2:k})}_{u}} \left(\dd\CV{\bX^{(l_1)}}{\bX^{(l_2)}}_u\right) \\
\Qua(\kap; \ell)_{t,T} &= \frac{1}{k!}\sum_{m =
  2}^{k}\binom{k-1}{m-1}\int_t^{T}\Big(\ad{\kap^{(\ell_{2:m})}_{u-}} \odot \ad{\kap^{(\ell_{m:k})}_{u}} \Big)\Big(\dd \outerbracket{\kap^{(l_{1})}}{
              \kap^{(l_{2})}}_u\Big)\\
\Cov(\bX, \kap; \ell)_{t,T} &= 1_{\{k\ge2\}}\frac{1}{(k-1)!} \int_t^T \left( \mathrm{Id} \odot \ad{\kap^{(l_{2:k})}_{u-}} \right) \left(\dd \outerbracket{\bX^{(l_1)}}{\kap^{(l_{2})}}_u\right) \\
\mathrm{HQua}(\kap; \ell)_{t,T} &= \int_t^{T} \sum_{j=2}^{k}\frac{B_{k-j}}{(k-j)!}\ad{\kap^{(\ell_{j:k})}_{u-}}\left(\dd\Qua(\kap; \ell_{0:j})_{u}\right)\\
\mathrm{HCov}(\bX, \kap; \ell)_{t,T} &= \int_{t}^{T}\sum_{j=1}^{k}\frac{B_{k-j}}{{(k-j)}!}\ad{\kap^{(\ell_{j:k})}_{u-}}\big(\dd\Cov(\bX, \kap; \ell_{0:j})_u\big) 
\end{align*}
\end{corollary}

\begin{example}
In the Gaussian martingale case from \Cref{exmpl:esig_recursion_martingale}, i.e., where $\bX = (0, \int_0^\cdot \sigma(s). \dd{B_s}, 0, \dots) \in \HSe^{-\infty}(\tf)$, an induction yields that $\kap^{(n)}$ is deterministic for all $n\in\N$.
Indeed, firstly the $\mathrm{HMag}^1$-terms vanish due to martingality of $\bX$. Secondly,
from the an induction hypothesis $\kap^{(0,n-1)}$ is deterministic, hence, of finite variation and all cross-variation terms vanish.
Finally, the remaining $\mathrm{HMag}^2$-terms are deterministic due $\QV{\bX}_t = \sigma.\sigma^T(t)$.
The recursion thus dramatically simplifies to $\kap^{(n)} \equiv 0$ for odd $n$ and
\begin{align*}
\kap^{(n)}_t = \sum_{\Vert \ell\Vert = n-2 } \frac{B_{|\ell|}}{|\ell|!} \int_t^{T} \ad{\kap^{(\ell)}_{u}}\big( \sigma.\sigma^T(u) \big) \dd{u}, \qquad\text{ for even }n.
\end{align*}
This is precisely the (deterministic) Magnus expansion of the logarithm of the solution of the equation \eqref{eq:esig_gaussian martingale}.
We will revisit this connection in \Cref{sec:levy} in the more general setting of time dependent Lévy-processes.

Note that $\kap_t$ is a Lie-series over symmetric 2-tensors $\mathrm{Sym}(\Rd \otimes \Rd) \subset \tf$. 
In the standard Brownian case $\sigma \equiv \I_d$, all commutators vanish and we obtain Fawcett's formula in logarithmic form $\kap_t = \frac{1}{2}(T-t) \I_d$.
\end{example}

\subsection{Diamond Products}\label{sec:diamond_products}
We extend the notion of the {\em diamond product} introduced in \cite{alos2018exponentiation}, which we recall, for continuous scalar semimartingales to our setting.
Denote by $\E_t$ the conditional expectation with respect to the sigma algebra $\F_t$.
\begin{definition}
  For \(X,Y\in\Sec(\mathbb{R})\) define
  \[
    (X\diamond Y)_t(T) \coloneq \E_t\left(\QV{X,Y}_{t,T}\right).
  \]
\end{definition}
For \(\TT\)-valued continuous semimartingales we have the following
\begin{definition} \label{def:diamondSym}
 For \(\bX\) and \(\bY\) in \(\Sec(\TT)\) define
$$
(\bX \diamond \bY)_t(T) \coloneq \E_t  \big( \QV{\bX, \bY}_{t,T}
\big)=\sum_{w\in\W_d}\left( \sum_{w_1w_2=w}(\bX^{w_1}\diamond\bY^{w_2})_t(T) \right)e_w \in \TT
$$
whenever the $\TT$-valued quadratic covariation which appears  on the right-hand side is integrable.
We also define an \emph{outer diamond} by
\[
  (\bX\blackdiamond\bY)_t(T)\coloneq\E_t(\outerbracket{\bX}{\bY}_{t,T})=\sum_{w_1,w_2\in\W_d}(\bX^{w_1}\diamond\bY^{w_2})_t(T)e_{w_1}\otimes
  e_{w_2}\in\TT\otimes\TT.
\]
\end{definition}
The following lemma allows us to rewrite all recursion from the previous sections in terms of diamond products.
\begin{lemma}[ {\cite[Lemma~2.2]{FHT22}} ] \label{lem:dm}
	Let $p,q,r\in[1,\infty)$ such that $1/p + 1/q + 1/r < 1$ and let $X\in\Ma_{\loc}^{c}((\Rd)^{\otimes l})$,  $Y\in\Ma_{\loc}^{c}((\Rd)^{\otimes m})$, and $Z\in\D((\Rd)^{\otimes n})$ with $l,m,n\in\N$, such that $\Abs{X}_{\HSe^{p}}, \Abs{Y}_{\HSe^{q}}, \Abs{Z}_{\Se^{r}} <\infty$. Then it holds for all $0 \le t \le T$
\begin{align*}
\E_t\left(\int_t^{T}Z_{u}\dd(X \diamond Y)_u(T)\right) = -\E_t\left(\int_{t}^T Z_{u}\dd\CV{X}{Y}_u\right).
\end{align*}
\end{lemma}
The proof relies on standard approximation results for continuous semimartingales as it is fairly clear from the definition of the diamond product that the identity holds for simple processes \(Z\), since \((X\diamond Y)_T(T)=0\).
The Kunita-Watanabe inequality ensures that the expectation on the right-hand side is well defined.

For clarity we focus on the case where $\bX = \bM \in \HSe^{\infty-}$ is a martingale, for otherwise one needs to carry along conditional expectations of the finite variation parts.
For the conditional expected signature, the recursion from \Cref{cor:recursion_mu} is then conveniently rewritten into
\begin{align*}
  \esig_t^{(n)} = \sum_{k=0}^{n-1} \E_t\bigg\{\frac12\sum_{i=1}^{n-k}\int_t^T\dd\big(\bM^{(i)}\diamond\bM^{(n-k-i)}\big)_u(T) \,\esig_u^{(k)}\bigg\} + (\bM^{(n-k)}\diamond\esig^{(k)})_t(T).
\end{align*}
Similarly, for the signature cumulants we obtain
\begin{align*}
    &\E_t\big(\mathrm{HMag}^2(\bX, \kap; \ell)_{t,T}\big) = \frac{B_{k-2}}{(k-2)!} \int_t^{T} \ad{\kap^{(\ell_{2:k})}_{u}} \Big(\dd(\bX^{(l_1)}\diamond\bX^{(l_2)})_u(T)\Big),\\
&\E_t\big(\Qua(\kap;\ell)_{t,T}\big)\\
    &\quad=-\E_t\Biggl\{\frac1{k!}\sum_{m=2}^k\binom{k-1}{m-1}\int_t^T\left(
      \ad\kap_{u-}^{(\ell_{2:n})}\odot\ad\kap_{u-}^{(\ell_{m:k})}
  \right)\left( \dd(\kap^{(\ell_1)}\blackdiamond\kap^{(\ell_2)})_u(T) \right)\Biggr\}.
\end{align*}
and
$$
  \E_t\big(\Cov(\bX,\kap;\ell)_{t,T}\big)=-\E_t\left\{
  \frac1{(k-1)!}\int_t^T\left( \Id\odot\ad\kap_{u-}^{(\ell_{3:k})}
\right)\left( \dd(\bX^{(\ell_1)}\blackdiamond\kap^{(\ell_2)})_u(T) \right) \right\}.
$$

When \(d=1\) (or in the projection onto the symmetric algebra, c.f. \Cref{sec:multivariate}) the
cumulant recursion takes a particularly simple form, since \(\ad\bx\equiv 0\) for all \(\bx\in\tf\).
\Cref{eq:recursion_h_form} then becomes
$$
  \kap^{(n)}_t (T) =
  \E_t\left(\bX_{t,T}^{(n)}\right)+\frac12\sum_{k=1}^{n-1} ( (\bX^{(k)} + \kap^{(k)}  ) \diamond (\bX^{(n-k)} + \kap^{(n-k)}))_t(T).
  $$
We shall revisit
this in a multivariate setting and comment on related works in \Cref{sec:multivariate}.

\subsection{Remark on tree representation}
As illustrated in the previous section, in the case where \(d=1\) (or when projecting onto the symmetric algebra cf. \Cref{sec:multivariate}),
the cumulant recursion takes a particularly simple form. The algebraic perspective in the setting of
 Friz, Gatheral and Radoi\c{c}i\'c \cite{friz2020cumulants} gives a tree series expansion of cumulants using binary trees.
 This representation follows from the fact that the diamond product of semimartingales is commutative
but not associative. As an example (with notations taken from \Cref{sub:diamond}), in case of a one-dimensional continuous martingale, the first terms are
\[
	\mathbb
	K_t(T)=\Forest{[]}+\frac12\Forest{[[][]]}+\frac12\Forest{[[[][]][]]}+\frac12\Forest{[[[[][]][]][]]}+\frac18\Forest{[[[][]][[][]]]}+\dotsb
\]
This expansion is organized (graded) in terms of the number of leaves in each tree, and each leaf
represents the underlying martingale.

In the deterministic case, tree expansions are also known for the Magnus expansion 
\cite{iserles1999solution} and the BCH formula \cite{CM2009}.
These expansions, also in terms of binary trees, are {\em different} from the ones above as the trees are required to be {\em planar} to account for the non-commutativity of the Lie algebra.
As an example, consider a matrix Lie group \(G\) with Lie algebra \(\mathfrak{g}\). Let \(Y\) solve the matrix ODE \(\dot{Y}_t=A_tY_t\) for some \(\mathfrak{g}\)-valued path \(A\), and set \(\Omega_t(T)=\log(e^{Y_t}e^{-Y_T})\). We have
\[
	\Omega_t(T) =
	\Forest{[]}+\frac12\Forest{[[[]][]]}+\frac{1}{12}\Forest{[[][[[]][]]]}+\frac{1}{4}\Forest{[[[[[]][]]][]]}+\dotsb
\]
In this expansion, the nodes represent the underlying vector field and edges represent integration (with respect to time)
and application of the Lie bracket, coming from the \(\ad\) operator. It is an interesting open question to find a unified tree representation that accounts for the unified functional recursion of our Corollary \ref{cor:recursion_h_form}.

\section{Multivariate Moments and Cumulants}\label{sec:multivariate}

We saw that $\T \coloneq 
T\dparl\mathbb{R}^d\dparr$ is the natural state space for  signatures, expected signatures and their logarithms. When $d=1$, the signature of a path $X:[0,T]$ is nothing more than the sequence
\[ \left(1, X_{0,T}, \frac{(X_{0,T})^2}{2},\frac{(X_{0,T})^3}{3!},\dots \right) \equiv 1 + X_{0,T} + \frac{(X_{0,T})^2}{2}+\frac{(X_{0,T})^3}{3!}+\dots
\]
The expected signatures is then exactly the sequence of moments of the random variable $X_{0,T}$, to the extent of being well-defined and up to factorial constants. Similarly, the signatures cumulants correspond to the sequence of classical cumulants of  $X_{0,T}$. Since $T\dparl\mathbb{R}^d\dparr$ is a commutative algebra if (and only if) $d=1$, our previous expression for expected signatures and signatures cumulants simplify dramatically, without becoming trivial (as pointed out in several works \cite{lacoin2019probabilistic, alos2018exponentiation,friz2020cumulants,Fuk21}). We can capture multivariate moments and cumulants, by working with the ``commutative shadow'' \cite{amendola2019varieties} of $\T$ which we now introduce.

\subsection{The symmetric algebra}\label{sec:symmetric_tensor_algebra}
The {\it symmetric algebra} over \(\Rd\), denoted by \(S(\Rd)\) is the quotient of \(T(\Rd)\) by the
two-sided ideal \(I\) generated by \(\{xy-yx:x,y\in\Rd\}\).
A linear basis of $S(\Rd)$ is then given by $\{ \hat e_v \}$ over non-decreasing words,
$v=(i_1,\dotsc,i_n) \in \widehat \W_d$, with $1 \le i_1 \le \dots \le i_n \le d, n \ge 0$. 
Every $\hat\bx \in S(\R^d)$ can be written as finite sum,
$$
    \hat\bx = \sum_{v \in \widehat \W_d} \hat\bx^v \hat e_v ,
$$
and we have an immediate identification with polynomials in $d$ commuting indeterminates. 
The canonical projection
\begin{equation}\label{eq:pisym}
	\pisym:T(\Rd)\twoheadrightarrow S(\Rd), \quad\bx \mapsto  \sum_{w\in\mathcal{W}_d} \bx^{w}\hat{e}_{\hat w},
\end{equation}
where $\hat{w}\in\hat{\mathcal{W}}_d$ denotes the non-decreasing reordering of the letters of the word $w\in\mathcal{W}_d$, is an algebra epimorphism, which extends to an epimorphism \(\pisym: \TT\twoheadrightarrow\Sy\) where $\Sy = S \dparl \R^d\dparr$ is the algebra completion, identifiable as formal series in $d$ commuting indeterminates. 
As a vector space, $\Sy$ %
can be identified with {\it symmetric} formal tensor series. Denote by $ \Sy_0$ and $\Sy_1$ the affine space given by those $\hat\bx\in\Sy$ with \( \hat\bx^\emptyset=0\) and \( \hat\bx^\emptyset=1\) respectively.
The usual power series in $\Sy$ define $\hatexp{}\colon \Sy_0 \to \Sy_1$ with inverse
$\hatlog{}\colon \Sy_1 \to \Sy_0$ and we have
\begin{align*}
\pisym\exp{(\bx + \by)} &= \hatexp{}(\hat \bx)\hatexp{}(\hat \by), \quad \bx, \by \in \tf\\
\pisym\log{(\bx \by)} &= \hatlog{}(\hat \bx) + \hatlog{}(\hat \by), \quad \bx,\by \in \TT_1.
\end{align*}
We shall abuse notation in what follows and write $e^{(\cdot)}$ (resp. $\log$), instead of $\hatexp$ (resp. $ \hatlog$).

All definitions for tensor valued continuous (and \cadlag) semimartingales have a straightforward extension to $\Sy$-valued process.
In particular, given \(\bX\) and \(\bY\) in \(\Sec(\Sy)\), the {\em inner} quadratic covariation is given by
\[
  \langle\bX,\bY\rangle
  =\sum_{w_1,w_2\in \widehat \W_d}\langle \bX^{w_1},\bY^{w_2}\rangle\hat e_{w_1}\hat e_{w_2} .
\]
Write $\Sy^N$ for the truncated symmetric algebra, linearly spanned by $\{ \hat e_{w}: w \in \widehat \W_d, |w| \le N\}$
and $\Sy^N_0$ for those
elements with zero scalar entry. 
In complete analogy with non-commutative setting discussed above, we then write $\HSehomSym^{q,N} \subset \Sec(\Sy^N_0)$ for the corresponding space of homogeneously $q$-integrable semimartingales.

Finally, also the definition of diamond products from Section~\ref{sec:diamond_products} extends immediately to $\Sy$-valued semimartingales.
In particular, given \(\hat\bX\) and \(\hat\bY\) in \(\Sec(\Sy)\), we have
$$
         ( \hat\bX \diamond \hat\bY)_t(T) \coloneq \E_t \big( \langle \hat\bX,  \hat\bY \rangle_{t,T} \big) = \sum_{w_1,w_2\in
  \widehat{\mathcal{W}}_d}  (\hat\bX^{w_1} \diamond \hat\bY^{w_2})_t(T)
  \hat e_{w_1}\hat e_{w_2}  \in \Sy,
$$
where the last expression is given in terms of diamond products of scalar semimartingales.

\subsection{Moments and cumulants} \label{sec:mc}
We quickly discuss the development of a symmetric algebra valued semimartingale, more precisely \(\hat\bX  \in\Sec(\Sy_0)\), in the group $\Sy_1$.
That is, we consider
\begin{equation}
    \dd \hat\bS = \hat\bS\,\circ \dd \hat\bX.
    \label{eq:gSigSy}
\end{equation}
It is immediate (validity of chain rule) that the unique solution to this equation, at time $t \ge s$, started at $\hat\bS_s = \hat\bs \in \Sy_1$ is given by
\[
    \hat\bS_{t}\coloneq e^{ \hat\bX_t- \hat\bX_s}\hat\bs \in \Sy_1
\]
and we also write $\hat\bS_{s,t} = e^{ \hat\bX_t- \hat\bX_s}$ for this solution
started at time $s$ from \(1\in\Sy_1\). The relation to signatures is as follows. 
\begin{proposition} \label{prop:XandXhat}
Let $\bX$ and $\bY$ be $\TT$-valued semimartingales and define  $\hat\bX \coloneqq \pisym(\bX)$ and $\hat\bY \coloneqq \pisym(\bY)$. 
Then for all $0 \le s \le t \le T$ it holds almost surely
\begin{equation}       
   \pisym \int_{s}^{t} \bX \mathrm d\mathbf Y= \int_s^t \hat\bX\,\mathrm d\hat{\mathbf Y},
    \label{eq:intproj}
\end{equation}
and
$$\pisym{\Sig(\bX)_{s,t}} = \hat\bS_{s,t} = e^{ \hat\bX_t- \hat\bX_s}.$$
\end{proposition}
\begin{proof}
    (i) That the projections \(\hat\bX,\hat{\mathbf Y}\) define \(\Sy\)-valued semimartingales follows
    from the componentwise definition and the fact that the canonical projection is linear.
    In particular, the right-hand side of \cref{eq:intproj} is well defined.
		To show \cref{eq:intproj} we apply the canonical projection \(\pisym\) to both sides of \cref{eq:def_tensor_ito_integral} after choosing \(Z_t\equiv\mathbf 1\), and using the explicit action of \(\pisym\) on basis tensors we obtain the identity
		\[
			\pisym\int\bX\,\mathrm d\bY=\sum_{w\in\W^d}\left( \sum_{uv=w}\int \bX^u\,\mathrm d\bY^v \right)\hat e_{\hat w}=\int\hat\bX\,\mathrm d\hat\bY.
		\]
		by \cref{eq:pisym}.
    Part (ii) is then immediate.
\end{proof}

Assuming componentwise integrability, we then define the conditional {\it symmetric moments} and {\it cumulants} of the $\Sy$-valued semimartingale $\hat\bX$ by
\begin{align*}
\hat\fmu_t(T) & \coloneq \E_t\left(e^{ \hat\bX_T- \hat\bX_t}\right) 
  \in \Sy_1,\\
\hat\kap_t(T) & \coloneq \log\left( \hat\fmu_t(T) \right)\in \Sy_0,
\end{align*}
for $0\le t \le T$.
If $\hat\bX = \pisym(\bX)$ for a $\TT$-valued semimartingale $\bX$, with expected signature and signature cumulants $\fmu(T)$ and $\kap(T)$, it is then clear that the symmetric moments and cumulants of $\hat \bX$ are obtained by projection,
$$\hat\fmu(T) = \pisym( \fmu(T)), \quad \hat\kap(T) = \pisym(\kap(T)).$$

\begin{example}    Let \( X\) be an \(\R^d\)-valued martingale in \(\mathscr H^{\infty-}\), and $\hat\bX_t\coloneq\sum_{i=1}^dX^i_t\hat e_i$. Then
    \begin{align*}
			\hat \fmu_t(T) &=\sum_{n=0}^\infty\frac{1}{n!}\E_t[(X_T-X_t)^n] \\
   &=1+\sum_{n=1}^\infty\frac{1}{n!}\sum_{i_1,\dotsc,i_n=1}^d\E_t\left[ (X_T^{i_1}-X_t^{i_1})\dotsm(X_T^{i_n}-X_t^{i_n}) \right]\hat e_{\widehat{i_1\dotsm i_n}},
    \end{align*}
		consists of the (time-\(t\) conditional) multivariate moments of $X_T-X_t \in \R^d$. Here, the series on the right hand side is understood in the formal sense.
	 It readily follows, also noted in \cite[Example 3.3]{bonnier2019signature},  that
   $\hat \kap_t (T) = \log(\hat\fmu_t(T))$ consists precisely of the multivariate cumulants of
   $X_T-X_t$. Note that the symmetric moments and cumulants of the scaled process $a X$, $a \in \R$,
   is precisely given by $\delta_a \hat\fmu$ and $\delta_a \hat\kap$ where the linear dilation map is
   defined by $\delta_a\colon \hat e_w \mapsto a^{|w|} \hat e_w$. The situation is similar for $a
	 \cdot X=(a_1X^1,\dotsc,a_dX^d)$, $a \in \R^d$, but now with $\delta_a\colon\hat e_w \mapsto a^w \hat e_1^{|w|}$ with $a^w = a_1^{n_1} \cdots a_d^{n_d}$ where $n_i$ denotes the multiplicity of the letter $i \in \{1,\dots, d\}$ in the word $w$.
\end{example}
We next consider linear combinations, $\hat\bX = a X + b \langle X \rangle $, for general pairs $a,b \in \R$, having already dealt with $b=0$. The special case $b = - a^2/2$, by scaling there is no loss in generality to take $(a,b) = (1,-1/2)$, yields a (at least formally) familiar exponential martingale identity.
\begin{example}
    Let \( X\) be an \(\R^d\)-valued martingale in \(\mathscr H^{\infty-}\), and define
    \[
        \hat\bX_t\coloneq\sum_{i=1}^dX^i_t\hat e_i-\frac12\sum_{1\le i\le j\le d}\langle
        X^i,X^j\rangle_t\hat e_{ij}.
    \]
    In this case we have trivial symmetric cumulants, \(\hat\kap_t(T)=0\) for all \(0\le t\le T\).
    Indeed, It\^o's formula shows that \(t\mapsto\exp(\hat\bX_t)\) is an
    \(\Sy_1\)-valued martingale, so that
    \[
    \hat\fmu_t(T)=\E_t[e^{\hat\bX_T-\hat\bX_t}]=e^{-\hat\bX_t}\E_t[ e^{\hat\bX_T}]=1.\]
    \end{example}
While the symmetric cumulants of the last example carries no information, it suffices to work with
$$
        \hat\bX = \sum_{i=1}^d a^i X^{i}\hat{e}_i  + \sum_{1\le i\le j\le d} b_{jk} \langle X^j,X^k \rangle \hat{e}_{ij}
$$
in which case $\hat\fmu = \hat\fmu (a,b), \hat\kap = \hat\kap(a,b)$ contains full information of the joint moments of $X$ and its quadratic variation process. A recursion of these was constructed as diamond expansion in \cite{friz2020cumulants}.

\subsection{Diamond relations for multivariate cumulants}
\label{sub:diamond}

We will demonstrate how a symmetrization of the functional equations from \Cref{sec:functional} and the recursions from \Cref{sec:reucsrion} lead to generalized view on the cumulant recursions from \cite{alos2018exponentiation,lacoin2019probabilistic,friz2020cumulants}.
As before we will first divide the focus between the continuous case and the discrete setting.
Referring to \cite[Section~5.2]{FHT22} for a unification of both settings into a general \cadlag form.

\subsubsection{The continuous case}
We assume that the usual assumptions from the continuous setting in \Cref{sec:continuous_case} on the filtration are in place.
Then following \Cref{def:diamondSym} the diamond product of $\hat \bX, \hat \bY \in \Sec(\Sy_0)$ is another continuous $\Sy_0$-valued semimartingale given by
$$(\hat \bX \diamond \hat \bY)_t(T) = \E_t \big( \langle \hat \bX, \hat \bY \rangle_{t,T} \big)
=
\sum  ( \E_t \langle \hat \bX^{w_1}, \hat \bY^{w_2} \rangle_{t,T})  \hat e_{w_1} \hat e_{w_2}, $$
with summation over all $w_1,w_2 \in \widehat \W_d$, provided all brackets are integrable. This trivially adapts to
$\Sy^N$-valued semimartingales, $N\in \mathbb{N}_{\ge 1}$, in which case all words have length at most $N$, the summation is restricted accordingly to $|w_1|+|w_2| \le N$.
\begin{theorem}\label{thm:main_with_jumps_sym}
 Let $\Xi= (0, \Xi^{(1)},\Xi^{(2)},...)$ be an $\F_T$-measurable random variable with values in $\Sy_0$, componentwise in $\mathcal{L}^{\infty-}$.
Then
\[
  \KK_t(T) \coloneq  \log \E_t\big( e^{\Xi}\big)
\]
satisfies the following functional equation, for all $0 \le t \le T$,
\begin{equation}\label{eq:main_with_jumps_sym}
   \KK_t(T)  = \E_t \Xi + \frac{1}{2} (\KK \diamond \KK)_t(T).
\end{equation}
Furthermore, if $N\in \mathbb{N}_{\ge 1}$, and $\Xi=(\Xi^{(1)},...,\Xi^{(N)})$ is $\F_T$-measurable with graded integrability condition
\begin{equation} \label{ref:Ncond}
    \Abss{\Xi^{(n)}}_{\Lcal^{N/n}} < \infty, \qquad n=1,...,N,
\end{equation}
then the identity (\ref{eq:main_with_jumps_sym})
holds for the cumulants upto level $N$, i.e.  for $\KK^{(0,N)} \coloneq \log(\E_t (e^{\Xi^{(0,N)}})$
with values in $\Sy^{N}_0$.
\end{theorem}
\begin{remark} Identity (\ref{eq:main_with_jumps_sym}) is reminiscent to the quadratic form of the generalized Riccati equations for affine diffusions. 
The relation can be presented more explicitly when the involved processes are assumed to have a Markov structure and the functional signature cumulant equation reduces to a PDE system (see \cite[Section~6.2.2]{FHT22}).
The framework described here, however, requires neither Markov nor affine structure. 
In \cite[Section~6.3]{FHT22} it is exemplified for affine Volterra processes that such computations are also possible in the fully non-commutative setting.
\end{remark}
\begin{proof}
    We first observe that since \(\Xi\in\mathcal L^{\infty-}\), by Doob's maximal inequality and the
    BDG inequality, we have that \(\hat\bX_t\coloneq\E_t\Xi\) is a continuous martingale in
    \(\HSe^{\infty-}(\Sy_0)\).
    In particular, thanks to \Cref{thm:bdg_signature}, the signature moments are well defined.
    According to \Cref{sec:mc}, the signature is then given by
    \[
    \Sig(\hat\bX)_{t,T}=e^{\Xi-\E_t\Xi},
    \]
    hence \(\hat\kap_t(T)=\KK_t(T)-\hat\bX_t\).
    
    Projecting \Cref{eq:master_sigcum} onto the symmetric algebra yields
    \begin{align*}
        \hat\kap_t(T) 
        &= \E_t\bigg\{ \hat\bX_{t,T}+\frac12\langle\hat\bX\rangle_{t,T}+\frac12\langle\kap(T)\rangle_{t,T}+\langle\hat\bX,\kap(T)\rangle_{t,T}\bigg\} \\
        &= \E_t\bigg\{\Xi+ \frac12\langle\KK(T)\rangle_{t,T} \bigg\}-\hat\bX_t,
    \end{align*}
    and \cref{eq:main_with_jumps_sym} follows upon recalling that $(\KK\diamond\KK)_t(T)=\E_t\langle\KK(T)\rangle_{t,T}$.
    The proof of the truncated version is left to the reader.
\end{proof}

As a corollary, we provide a general view on recent results of \cite{alos2018exponentiation,lacoin2019probabilistic,friz2020cumulants}.
\begin{corollary}
\label{thm:diamond_recursion}
The conditional multivariate cumulants $(\KK_t)_{0\le t\le T}$ of a random variable \(\Xi\) with values in
\(\Sy_0(\R^d)\), componentwise in \(\mathcal L^{\infty-}\) satisfy the recursion
\begin{equation}\label{eq:diamond_recursion}
  \KK^{(1)}_t = \E_t(\Xi^{(1)}) \quad \text{and} \quad \KK^{(n)}_t =
  \E_t(\Xi^{(n)})+\frac{1}{2}\sum_{k=1}^{n}\left( \KK^{(k)} \diamond \KK^{(n-k)}\right)_t(T) \quad \text{ for } n \ge 2,
\end{equation}
The analogous statement holds true in the $N$-truncated setting, i.e. as recursion for $n=1,..,N$ under the condition (\ref{ref:Ncond}).
\end{corollary}

\begin{remark} In absence of higher order information, i.e., $\Xi^{(2)} = \Xi^{(3)} = ... \equiv 0$, this type of cumulant recursion appears in \cite{lacoin2019probabilistic}; and under optimal integrability conditions on $\Xi^{(1)}$ with finite $N$th moments in \cite{friz2020cumulants}. (The latter requires a localization argument which is avoided here by directly working in the correct algebraic structure.)
\end{remark}

\subsubsection{The discrete case}\label{sec:diamonds_discrete}

Consider a probability space $(\Omega, \G,\PM)$  with discrete filtration $(\G_j)_{j = 0, \dots, J}$ and a $\G_T$-measurable $\Sy_0$-valued random variable $\Xi$. 
Assuming that $\Xi$ is componentwise in $\mathcal{L}^{\infty-}$, we are again interested in calculating $(\KK_j) \coloneq  (\log \E(e^{\Xi}\,\vert \G_j))$.
Instead of projecting the identities from  the non-commutative setting in \Cref{sec:cadlag_case} to the symmetric algebra, we present a direct derivation of the corresponding identities.

Using the multiplicative property of the exponential map in the symmetric algebra, the martingale property of $e^{\KK}$ can be rewritten to
$$ 0 = \E\big( e^{\KK_{j+1} - \KK_j} -1 \,\big\vert \G_j \big).$$
A summation of the above identity paired with $$\KK_j = \Xi - (\KK_J - \KK_j) = \Xi - \sum_{i=j}^{J-1} (\KK_{i+1} - \KK_i)$$ 
and conditioning to $\G_j$ then yields the following.

\begin{theorem}\label{thm:sym_disc_cumulants}
Given that the $\G_J$-measurable $\Sy_0$-valued random variable is componentwise in $\mathcal{L}^{\infty-}$.
It follows that $(\KK_j) \coloneq  (\log \E( e^{\Xi}\,\vert \G_j))$ satisfies
for $j = J, J-1, \dots, 0$:
\begin{align}
    \KK_j = \E(\Xi \,\vert\G_j) + \E\bigg( \sum_{i=j}^{J-1} e^{\KK_{i+1} - \KK_i} -1 - (\KK_{i+1} - \KK_i)\,\bigg\vert \G_j \bigg).
\end{align}
In particular, projecting to symmetric tensor levels we see that the conditional cumulants of $\Xi$ satisfy for all $n = 1, 2, \dots$ and $j = J, J-1, \dots, 1$:
 \begin{align*}
 \KK_j^{(n)}  = \E\big( \Xi^{(n)} \,\vert \G_j) + \E \bigg( \sum_{\Vert \ell \Vert = n, \vert \ell \vert = k}\frac{1}{k!}(\KK_{i+1} - \KK_{i})^{(\ell)} \;\bigg\vert \G_j \bigg),
 \end{align*}
\end{theorem}

This is really the result we get from projecting the \Cref{thm:discrete} to the symmetric algebra (upon using that the symmetric signature cumulants are given by $\hat\kap_j = \KK_j - \E(\Xi \,\vert \G_j)$).
Nevertheless, due the ad hoc derivation presented here, it might seem surprising that the resulting expansion leads  to any non-trivial relations for conditional cumulants at all.

While on the first level one still trivially has $(\KK^{(1)}_j) = (\E(\Xi^{(1)}\vert \mathcal{G}_j))$, on the second level we have
\begin{align*}
        \KK_j^{(2)} &= \E(\Xi^{(2)}\,\vert\G_j)+ \frac12\E \bigg( \sum_{i=j}^{J-1}
        \big(\KK^{(1)}_{j+1}-\KK^{(1)}_j\big)^2 \;\bigg\vert \G_j\bigg) %
\end{align*}
which one recognizes, in case $\Xi^{(2)} = 0$, as the energy identity for the discrete square-integrable martingale $(\KK^{(1)}_j) = (\E(\Xi^{(1)}\,\vert\G_j))$.
Going further in the recursion yields increasingly non-obvious relations.
Taking $\Xi^{(2)} = \Xi^{(3)} = ... \equiv 0$ for notational simplicity gives
\begin{equation}\label{eq:sym_disc_cum3}
 \KK_j^{(3)}  = \mathbb{E} \bigg( \sum_{i=j}^{J-1} \frac{1}{6}(\KK^{(1)}_{i+1} - \KK^{(1)}_{i})^3 +  (\KK^{(1)}_{i+1} - \KK^{(1)}_{i})(\KK^{(2)}_{i+1} - \KK^{(2)}_{i}\big) \;\bigg\vert \G_j \bigg).
 \end{equation}
 It is interesting to note that inductively the identity for $\KK^{(n)}$ can by expressed in terms of variations (of variations) of the martingale $\KK^{(1)}$, which relates to the {\em Bartlett identities} for martingales that have appeared in the statistics literature, cf. Mykland \cite{Myk1994} and the references therein.
 \begin{remark}
     The second formula in \Cref{thm:sym_disc_cumulants}, in terms of summation over partitions $\ell$, is sub-optimal in the sense that due to commutativity there are many repeated terms.
     For example, the factor \(\tfrac12\) in front of the second term in \cref{eq:sym_disc_cum3} suggested by that formula is actually just 1, since both compositions \(\ell=(1,2)\) and \(\ell=(2,1)\) give the same product.
     In general one can take care of this algebraically by using (rescaled) Bell polynomials, which in the context of conditional cumulants was realized in \cite{Fuk21}.
 \end{remark}

\section{Applications} \label{sec:applications}

\subsection{Time-inhomogeneous Lévy processes}\label{sec:levy}

We now wish to derive a formula for the expected signature of time-inhomogeneous Lévy processes. This class of processes models trajectories having jumps, possibly infinitely many.
To avoid the technical difficulties entailed by general càdlàg\footnote{From French ``continu(e) à droit, limite a gauche'' (right-continuous with left limits).} stochastic integration, we will only consider time-inhomogeneous Lévy process that have finitely many jumps.
Specifically, for a filtered probability space $(\Omega, \F, \PM, (\F_t))$ satisfying the usual conditions we consider a process of the form
\begin{align}\label{eq:simple_levy}
    X = \int_0^\cdot b(u) \dd{u} ~+~ \sum_{k=1}^m\int_0^\cdot\sigma_k(u)\dd{B^k_u} ~+~ N,
\end{align}
where $b: [0,T] \to \R^d$ is integrable, $\sigma_k: [0,T] \to \R^d$ is square-integrable, $B = (B^1, \cdots, B^m)$ is a Brownian motion with respect to $(\F_t)$ and 
$N$ is a time-inhomogenous \textit{compound Poisson process} with respect to $(\F_t)$, i.e., $N$ is a process with piecewise constant right-continuous sample paths in $\Rd$ and
is fully characterized by the existence of a finite \textit{intensity measure} $K$ so that
\begin{align}\label{def:intensity_measure}
\E_s\bigg( \sum_{s < u \le t} f(\Delta N_u) \bigg) = \int_s^t \int_{\Rd} f(x) K_u(\dd x) \dd{u},
\end{align}
for all $0 \le s \le t \le T$ and bounded measurable functions $f:\Rd\to\R$, where
$$\Delta N_t \coloneq N_t - N_{t-} \coloneq N_t - \lim_{\varepsilon\to0}N_{t-\varepsilon}.$$

The interpretation of signature as the ``geometric'' development of a $\tf$-value semimartingale $\bX$ in the group $\TT_1$ extents from the continuous \cite{hakim1986exponentielle} to the \cadlag case \cite{estrade1992exponentielle} and is consistent with the Marcus \cite{marcus1978modeling,marcus1981modeling, kurtz1995strtonovich, applebaum2009levy, friz2017general}  interpretation of 
 \begin{align}\label{eq:sig_markus}
 \dd{\bS}_t = \bS_t \circ \dd{\bX}_t, \qquad t\ge s, \quad\bS_s = 1 \in \TT_1.
 \end{align}
 The definition of the signature $\Sig(\bX)_{s,\cdot}$ as the solution to \eqref{eq:sig_markus} results in the desired property that, whenever the process jumps, the signature is multiplied by the
exponential displacement, i.e.,
 $$\Sig(\bX)_{s,t} = \Sig(\bX)_{s,{t-}}e^{\Delta \bX_t}.$$
 
 For the time-inhomogeneous Lévy processes $X$ of the form \eqref{eq:simple_levy} we can resolve the above Marcus definition of the signature $\bS = \Sig(X)_{s,\cdot}$ by simply integrating between the finite number of jump times.
 This finally leads to the Itô-integral from
 \begin{multline}\label{eq:marcus_signature_ito}
    \bS_t = 1 + \int_s^t \bS_{u-}\left(b(u) + \frac{1}{2}a(u)\right) \dd{u} + \sum_{k=1}^m\int_s^{t} \bS_{u-}\, \sigma_k(u)\dd{B^k_u}\\
    + \sum_{s< u \le t} \bS_{u-}\big(e^{\Delta N_u}-1\big),
 \end{multline}
    where $a(t) := \sum_{k=1}^m (\sigma_k(t))^{\otimes 2} \in \mathrm{Sym}(\R^{d\times d}).$

We then have the following result for the expected signature and signature-cumulants of time-inhomogeneous Lévy processes.

\begin{proposition}\label{prop:levy}
Let $X$ be a process of the from \eqref{eq:simple_levy} with an intensity measure $K$ satisfying the moment condition
\begin{align}\label{eq:levy_moment_condition}
    \int_0^T \int_{\Rd}|x|^q {K}_t(\dd x) \dd{t}<\infty, \qquad q \ge 1,
\end{align}
Then its conditional expected signature
$\esig = ( \E_t(\Sig(X)_{t,T})_{0 \le t \le T}$ is the solution of the following backwards differential equation in $\TT_1$:
\begin{align}\label{eq:levy_esig}
    \esig_{t} = 1 + \int_t^T \mathfrak{y}(u)\, \esig_{u} \dd{u}, \qquad 0 \le t \le T,
\end{align}
where
\begin{equation}\label{def:levy_kintchin_exponent}
\etfrak(t)\coloneqq  b(t) + \frac{1}{2}a(t) + \int_{\Rd}(e^\bx -1) K_t(\dd \bx) \in \tf.
\end{equation}
Moreover, the conditional signature cumulants $\kap = \log( \esig)$ are the solution of
\begin{align*}\label{eq:inhom_levy_kintchin}
    \kap_t = \int_t^{T}H(\ad{\kap_{u}})(\etfrak(u))\,\dd u, \qquad 0 \le t \le T.
\end{align*}
\end{proposition}

\begin{remark}
The case of general time-homogeneous  Lévy process, i.e., semimartingales with independent increments that are continuous in probability, is treated in \cite[Corollary~6.5]{FHT22}.
Allowing for infinitely many jumps by extending $K$ to a Lévy-measure, i.e., by requiring $K$ to only be a $\sigma$-finite measure satisfying $\int_0^T\int_{\Rd}1\wedge|x|^2 K_t(\dd{x})\dd{t} < \infty$, a martingale compensation of the small jumps is required and results in  subtracted indicator function $\indik_{|x|>1}$ appearing in \eqref{eq:levy_moment_condition} and \eqref{def:levy_kintchin_exponent}.
\end{remark}
    
\begin{proof}
    To avoid using general stochastic integration, we present a direct proof following \cite{friz2017general}.
    This prove bases on the identity
    \begin{equation}\label{eq:levy_integral_identity__}
        \E_s\left[\sum_{s< u \le T} H_{u-}f(\Delta N_u)\right] 
        =  \E_s\left[\int_s^t \int_{\Rd} H_{u-}f(\Delta N_u) K_u(\dd x) \dd{u} \right],
    \end{equation}
    which is proven using a measure theoretic induction starting with \eqref{def:intensity_measure} for $H\equiv 1$ and all bounded measurable functions $f:\Rd\to\R$, and then extending over simple processes $H$, i.e. piecewise constant adapted bounded processes, to all càdlàg adapted process $H$ with $\Vert H \Vert_{\Sesup^1} < \infty$.
    We refer to \cite[Lemma~43]{friz2017general} for a more detailed proof in the time-homogeneous case.
    
    As before we will omit proving  the integrability properties of $\Sig(X)_{0,T}$.
    As proven in \cite[Corollary~6.5]{FHT22} the moment condition \eqref{eq:levy_moment_condition} suffices to show that $\Vert \Sig(X)_{0,\cdot}^w\Vert_{\Sesup^q} < \infty$ for all $w \in \W_d$ and $q >1$.
    Hence, the expected signature is well defined and it immediately follows from the independence of increments of $X$ with respect to the underlying filtration that $$\esig_s(t) = \E_s[\Sig(X)_{s,t}] = \E_s[\Sig(X)_{s,t}], \qquad 0 \le s \le t.$$
    Taking expectation of each of the terms in Itô-integral form \eqref{eq:marcus_signature_ito} of $\bS_t = \Sig(X)_{s,t}$ we obtain using the integrability of $\bS$ and the Fubini's theorem that firstly
    \begin{align*}
        \E_s\bigg[\int_s^t \bS_{u-}\left(b(u) + \frac{1}{2}a(u)\right) \dd{u}\bigg] =  \int_s^t \esig_t(u-)\left(b(u) + \frac{1}{2}a(u)\right) \dd{u}, 
    \end{align*}
    secondly using the martingality of the stochastic integral with respect to Brownian motion
    \begin{align*}
    \E_s\bigg[\sum_{k=1}^m\int_s^{t} \bS_{u-}\, \sigma_k(u)\dd{B^k_u} \bigg] = 0
    \end{align*}
    and finally using \eqref{eq:levy_integral_identity__} and once more Fubini's theorem
   \begin{align*}
        \E_s\left[\sum_{s< u \le T} \bS_{u-}\big(e^{\Delta N_u}-1\big)\right] 
        &=  \int_s^t \int_{\Rd} \esig_s(u-)\big(e^{\bx}-1\big) K_u(\dd x) \dd{u}.
    \end{align*}
    Summarizing, this show that $\esig_s$ satisfies the forward equation
    \begin{align*}
        \esig_s(t) = 1+ \int_s^t \esig_s(u-)\mathfrak{y}(u) \dd{u}, \qquad s \le t \le T.
    \end{align*}
    This implies in particular that $t\mapsto \esig_s(t)$ is absolutely continuous.
    Furthermore, by the uniqueness of the solution flow we also have the identity  $\esig_0(t)\esig_t(T) = \esig_0(T)$. 
    Hence, $t\mapsto \esig_t(T) = \esig_0(t)^{-1}\esig_0(T)$ is absolutely continuous as well, and differentiation of the previous identity yields
    $$\esig_0(t)\mathfrak{y}(t)\esig_t(T) + \esig_0(t)\frac{\dd}{\dd t}\esig_t(T) = 0,$$
    for almost every $t\in[0,T]$.
    Multiplication with $\esig_0(t)^{-1}$ and integration over $[t,T]$ then yields the backwards equation.
    
     The second part of the statement follows by noting that from the first part of the proof we have $$\E_t\big( \Sig(X)_{t,T} \big) = \Sig\Big(\int_0^\cdot \mathfrak{y}(u)\dd{u}\Big)_{t,T}.$$
     Hence, $\esig$ is really the signature of the continuous deterministic $\tf$-valued path $\bX := \int_0^\cdot \mathfrak{y}(u)\dd{u}$.
     An application of  \Cref{thm:main_sigcum} then directly leads to the stated expansion for the log-signature $\log(\Sig(\bX)_{t,T}) = \log(\esig_t(T))$.
\end{proof}
We conclude this section with an important consequence that concerns the \textit{convergence radius} of the expected signature, which for a general tensor series $\bx\in\TT$ is defined as the maximal $r \in [0,\infty)$ such that 
\begin{align}\label{eq:geometric_decay}
     \sum_{n=0}^\infty \lambda^n \vert \bx^{(n)}\vert < \infty, \qquad \text{for all } \lambda \in (0,r).
\end{align}
In \cite{chevyrev2016characteristic} it was proven that if the expected signature has infinite convergence radius it characterizes the law of the signature, hence, the law of the underlying process when the signature itself is characterizing.
Thus, rendering the question of infinite convergence radius a central but in general difficult question. For Brownian motion it is a consequence of Fawcett's formula \cite{fawcett2002problems}, for Lévy-process a consequence of the signature Lévy-Khinchin formula in \cite{friz2017general}. Apart from that there is a negative answer in \cite{boedihardjo2021expected, li2022expected} for the case of Brownian motion stopped at the exit of a domain.
The following results gives necessary condition for time-inhomogeneous $\mathfrak{g}^N$-valued Lévy processes.
\begin{corollary}
Let $X, \esig$ and $\mathfrak{y}$ be as in \Cref{prop:levy}.
Then for any $N \in \mathbb{N}$ and $\lambda > 0$ it holds for all $0 \le t \le T$:
\begin{align*}
    \sum_{n=0}^N \lambda^n \vert \esig^{(n)}_t\vert \le \exp\left(\sum_{n=1}^N \lambda^n \int_t^T \vert \mathfrak{y}^{(n)}(s)\vert \dd{s} \right).
\end{align*}
In particular, $\esig_0$ has infinite convergence radius if $\int_0^T |\mathfrak{y}(u)| \dd{u}$ has.
\end{corollary}
\begin{proof}
Projecting \eqref{eq:levy_esig} to tensor level $n\in\mathbb{N}$ and using the compatibility of the norm on $(\Rd)^{\otimes n}$ we have for any $\lambda > 0$ and all $t\in[0,T]$:
\begin{align*}
    \lambda^{n}\vert\esig^{(n)}_t\vert ~\le~ \sum_{k=0}^{n} \int_t^T\lambda^k\vert\esig^{(k)}_s\vert  \lambda^{n-k}\vert\mathfrak{y}^{(n-k)}(s)\vert \dd{s}
\end{align*}
Summation of the above inequality yields for all $t \in [0,T]$:
\begin{align*}
    \sum_{n=0}^N \lambda^{n}\vert\esig^{(n)}_t\vert
    ~\le&~ 1 + \sum_{n=0}^N \sum_{k=0}^{n} \int_t^T\lambda^k\vert\esig^{(k)}_s\vert  \lambda^{n-k}\vert\mathfrak{y}^{(n-k)}(s)\vert \dd{s} \\
    ~\le&~ 1 + \int_t^T \left(\sum_{n = 0}^N \lambda^n\vert\esig^{(n)}_s\vert \right) \left(\sum_{n = 0}^N \lambda^{n}\vert\mathfrak{y}^{(n)}(s)\vert \right) \dd{s}.
\end{align*}
The stated estimate then follows from Gronwall's lemma and the second claim by talking taking limits for $N\to\infty$.
\end{proof}

\subsection{Brownian rough paths}\label{sec:brownian_rough_paths}

This section is framed within rough path terminology and is directed to a reader with special interest into the topic.
The unfamiliar reader may use \cite{friz2020course} as a starting point.
We denote by $\mathcal{G}^M \subset \TT^M_1$ the free step-$M$ nilpotent free Lie group and recall that the space of weakly geometric rough paths consists of continuous finite $p$-variation paths $\mathcal{W}\Omega^{c,p}_T = C^p([0,T]; (\mathcal{G}^M, \Aabss{\cdot}))$ with $M \le p < M+1$ and where $\Aabss{\cdot}$ is a homogeneous norm on $\mathcal{G}^M$.
In accordance with \cite{friz2017general, chevyrev2018random} we call a random variable $\mathbf{Y}$ with values in $\mathcal{W}\Omega^{c,p}_T$ a
\emph{Brownian rough path with drift} if it has stationary and independent increments $\mathbf{Y}_{s,t} := \mathbf{Y}_{s}^{-1}\mathbf{Y}_{t}$.
For the means of this section we denote the signature of $\bY$, i.e., the full Lyons lift, by $\bY^{<\infty}$.

Here, we will specifically consider the case where $M = 2N$ and $\bY = \bB$ is the Stratonovich development of Brownian motion in the free step-$N$ nilpotent Lie algebra $\mathfrak{g}^N = \log\mathcal{G}^N$.
More precisely, let $\mathbb{B} = \sum_{i\in\mathcal{I}} B^{i} \mathfrak{u}_i$, where $\{\mathfrak{u}_i\}_{i\in \mathcal{I}}$ is a orthonormal system in $\mathfrak{g}^N \subset \tf^N$ and $(B^{i})_{i\in\mathcal{I}}$ is a $|\mathcal{I}|$-dimensional Brownian motion with possibly mutually correlated coefficients.
Define the correlation matrix of $\mathbb{B}$ by
\begin{align*}
    \Sigma := \sum_{i,j\in \mathcal{I}}\E(B_1^iB_1^j)\mathfrak{u}_i\mathfrak{u}_j \;\;\in \mathrm{Sym}(\mathfrak{g}^N \otimes \mathfrak{g}^N) \subset \TT_0^{2N}.
\end{align*}
The truncated Stratonovich development $\bB := \pi_{(0,2N)}\Sig(\mathbb{B})_{0,\cdot}$
takes values in the group $\mathcal{G}^{2N}$ and following standard arguments from \cite[Section~3.2]{friz2020course} we see that the samples of $\mathbf{B}$ are (weakly) geometric rough path for any $p \in (2N, 2N+1)$.
Clearly, this is still the case after adding a drift $\mathbb{B} \leadsto \mathbb{B}(\mathfrak{y}) := (\mathbb{B} + t\mathfrak{y})_{t\ge0}$ for any constant vector $\mathfrak{y} \in \mathfrak{g}^{2N}$.
The resulting development $\mathbf{B}(\mathfrak{y})$ is a Brownian rough-path with drift, as independence and stationarity of the increments follow directly from the definition.
Its signature is given by the full development $\bB(\mathfrak{y})^{<\infty} = \Sig(\mathbb{B}(\mathfrak{y}))_{0,\cdot}$.

\Cref{thm:main_esig} implies the following generalization of Fawcett's formula.
\begin{corollary}\label{cor:brp}
The expected signature of a Brownian rough path $\bB(\mathfrak{y})$ with correlation $\Sigma$ and drift $\mathfrak{y}$ is given by $$\E\big(\mathbf{B}(\mathfrak{y})^{<\infty}_T\big) = \exp\bigg(T\mathfrak{y} + \frac{T}2 \Sigma \bigg).$$
\end{corollary}
\begin{remark}
    The above suggest to define a ``standard''  Brownian rough path by requiring additionally that $\log\E(\mathbf{Y}_{s,t}) = \frac{1}{2}(t-s)\sum_{i\in\mathcal{I}}(\mathfrak{u}_i)^2$ and that $\{\mathfrak{u}_i\}$ spans $\mathfrak{g}^N$.
    However, this is only a appropriate definition as far as it is appropriate to call the Stratonovich lift the ``standard'' lift of a Brownian motion.
\end{remark}

\begin{remark}
The results from the previous section, can be extended to time-inhomogeneous \textit{Lévy rough paths} in the  sense of \cite{friz2017general, chevyrev2018random}, without any additional effort.
Indeed, instead of starting with the  differential characteristics $(b, a, K)$  in $\Rd$ we could have likewise used characteristics in $\mathfrak{g}^N$, resulting in a $\mathfrak{g}^N$-valued semimartingale $\bX$ with independent increments.
The truncated Marcus lift $\bS = \pi_{(0,2N)}\Sig(\mathbb{X})$ then has independent group increments $\bS_{s,t} = \bS_{s}^{-1}\bS_{t}$ and the sample paths are \cadlag (weakly) geometric rough paths in the sense of \cite{friz2017general, friz2018differential}.
\end{remark}
\begin{proof}
From the definition we then have $\bB(\mathfrak{y}) := \pi_{(0,2N)}\Sig(\mathbb{B}(\mathfrak{y}))_{0,\cdot}$ where $\mathbb{B}$ is Brownian motion in $\mathfrak{g}^N$. For any $w \in \mathcal{W}_d$ and $q \ge 1$ we have
$$\Vert \mathbb{B}(\mathfrak{y})^w \Vert_{\HSe^q}
= \Big\Vert \QVsmall{ \sum_{i} B^i \mathfrak{u}_i^w}_T^{1/2}+ T\vert \mathfrak{y}^w\vert \Big\Vert_{\Lcal^{q}}  =  \sqrt{T}\Big\vert\sum_{i,j\in\mathcal{I}}\Sigma^{ij}\mathfrak{u}^w_{i}\mathfrak{u}^w_j\Big\vert^{1/2}  + T\vert \mathfrak{y}^w\vert  < \infty.$$
Hence, $\mathbb{B}(\mathfrak{y}) \in \HSe^{\infty-}$ and thus \Cref{thm:main_esig} and \Cref{cor:recursion_mu} apply to the conditional expected signature $\esig_t := \E_t(\Sig(\mathbb{B}(\mathfrak{y}))_{t,T})$.
Due to martingality of $\mathbb{B}$ we have 
$$\esig^{(1)}_t = \E\Big(\mathbb{B}^{(1)}_{t,T} + (T-t)\mathfrak{y}^{(1)}\Big) =  (T-t)\mathfrak{y}^{(1)}$$
Further, due to the deterministic quadratic variation $\QV{\mathbb{B}}_t = t\Sigma$ one sees inductively by following the recursion in \Cref{cor:recursion_mu} that $\esig^{(n)}$ is deterministic and hence by \Cref{thm:main_esig} satisfies
\begin{align*}
\esig_t  =1 + \int_t^T \Big( \mathfrak{y}+ \frac{1}{2} \Sigma\Big)\esig_u \dd{u}.
\end{align*}
We readily conclude by noting that the unique solutions to the above ordinary differential equation in $\TT_1$ is given by $\esig_t = e^{(T-t)(\mathfrak{y}+ \frac12\Sigma)}$.
\end{proof}

We mention two important examples:
Firstly, the zero mass limit of a \emph{physical Brownian motion in a magnetic field} (\cite{friz2015physical}, \cite[Section~3.4]{FHT22},\cite{li2023small}), which is a Brownian rough path with $N=1$ and drift $\mathfrak{y} = (0, 0, A)$ where $A\in\mathfrak{so}(\Rd\otimes\Rd)$ is an antisymmetric matrix that depends on the physical environment.
In line with \cite{li2023small} the above result immediately gives that the expected signature of the small mass limit is given by $e^{T(A +  \frac12\Sigma)}$.

As a second example we mention the recent finding by Hairer \cite{hairer2024renormalization} that the $H\to \frac{1}{4}$ limit of a suitably renormalized fractional Brownian motion converges to a pure area rough path with coefficients given by standard Brownian motion.
This limit is naturally included in our framework as a Brownian rough path of $\mathbb{B} = \sum_{i\in\mathcal{I}}\mathfrak{u}_i B^i$ where the corresponding orthonormal system of $\mathfrak{g}^2$ is given by $\{\mathfrak{u}_i : i\in\mathcal{I}\} = \{[e_i, e_j] : 1 \le i < j \le d\}$.
\Cref{cor:brp} yields the corresponding generalization of Fawcett's formula.
We note that it might be of interest to obtain this formula as the rescaled limits for $H\to \frac{1}{4}$ in \cite{cass2024wiener}.

\section{Conclusion}\label{sec:conclusion}

We summarize the results of this chapter and comment on their practical relevance.
An overview of the different formulas and corresponding recursions with the relation to recent and classical results from the literature is displayed in Figures~\ref{fig:theorems} and~\ref{fig:cube.recur}. 

\begin{figure}[h]
    \centering
    \begin{tikzcd}[math mode=false, cells={align=center,text width=10em}, sep=huge]
        {\small Func.~Eq.~\eqref{eq:master_sigcum} $\Sec$-Sig.~Cum. }\ar["{\small commutative}"description]{dd}\ar["{\small deterministic}"description]{rr}&&{\small   Schur-Hausdorff~Eq. \eqref{eqHDODE} $\Fv^c$-Log-Sig.}\ar{dd}\\
        &&\\
        {\small Func.~Eq. \eqref{eq:main_with_jumps_sym}\break $\Sec$-Cumulants} \ar{rr}&&trivial\ar[from=uu,crossing over]
    \end{tikzcd}
    \caption{An overview of the different formulas and corresponding recursions (in Figure~\ref{fig:cube.recur}) with the relation to recent and classical results from the literature. Recall that $\Se^{c}$ stands for continuous semimartingales and $\Fv^{c}$ for continuous processes of finite variation. Here, ``trivial'' refers to the linear equation $\dd{\hat\bS}=\hat\bS\dd{\hat\bX}$ with $\hat\bS_0=1\in\Sy$ in the commutative case $\hat\bX\in\Fv^c(\Sy)$, which is solved by $e^{\hat\bX-\hat\bX_0}$ with logarithm $\hat\bX-\hat\bX_0$.} %
    \label{fig:theorems}
\end{figure}
\begin{figure}[h]
    \centering
    \begin{tikzcd}[math mode=false, cells={align=center,text width=10em}, sep=huge]
        {\small Recursion \eqref{eq:recursion_h_form} $\Sec$-Sig.~Cum. }\ar["{\small commutative}"description]{dd}\ar["{\small deterministic}"description]{rr}&&{\small  Magnus Expansion\break$\Fv^c$-Log.~Sig.}\\
        &&\\
        {\small Diamond \break Expansion \eqref{eq:diamond_recursion}} \ar{rr}&&trivial\ar[from=uu,crossing over]
    \end{tikzcd}
    \caption{Accompanying recursions to Figure~\ref{fig:theorems}.}
    \label{fig:cube.recur}
\end{figure}

\begin{itemize}
    \item Theorem~\ref{thm:main_esig} and Theorem~\ref{thm:main_sigcum}, more precisely equation \eqref{eq:esig_master} and \eqref{eq:master_sigcum}, respectively characterize expected signatures and signature cumulants of semimartingales by dynamic formulas, with respect to the underlying filtration.
    The functional equation \eqref{eq:master_sigcum} for signature cumulants is represented in the top-left corner of Figure~\ref{fig:theorems}.

    \item Projecting \eqref{eq:esig_master} and \eqref{eq:master_sigcum} to tensor levels we obtain recursions for signature moments and signature cumulants respectively which are presented in \Cref{cor:recursion_mu} and \Cref{cor:recursion_h_form}.
    Specifically, these recursions represent $\esig^{(n)}$ in terms of $\esig^{(1)}, \dots, \esig^{(n-1)}$ (and analogously for $\kap$).
    The recursion for signature cumulants is represented by the top-left corner in Figure~\ref{fig:cube.recur}.

    \item Even when an explicit form of $\esig_t$ or $\kap_t$ is obtainable, the recursive nature of the expressions for each homogeneous component can be useful in some numerical scenarios.
    This provides a means of computation which can in principle be more efficient than the naive Monte Carlo approach.
    This is most evidently the case in discretized Markovian situations (compare with \Cref{expl:markov}), where the recursion from \Cref{thm:discrete} directly turns into a inductive backwards scheme, where conditional expectations can numerically be approximated using functional linear regression procedures.
 
    \item The most classical consequence of Theorem~\ref{thm:main_sigcum} appears when $\bX$ is a deterministic continuous semimartingale, i.e.,  in particular the components of $\bX$ are continuous paths of finite variation.
    The signature cumulants are then just the log-signature $\kap_t(T) = \log\Sig{(\bX)}_{t,T}$ and what remains of \eqref{eq:master_sigcum} is a classical differential equation due to \cite{schur1891theorie,hausdorff1906symbolische}, here in backward form
    \begin{equation} \label{eqHDODE}
        - \dd \kap_t (T) = H(\ad{\kap_{t}})\dd\bX_t.
     \end{equation} 
    The accompanying recursion is then precisely Magnus expansion \cite{magnus1954exponential, iserles1999solution, iserles2005lie, blanes2009magnus}.
    This transition to the deterministic case is represented in Figure~\ref{fig:theorems} and~\ref{fig:cube.recur} by going from the top-left to the top-right corner.
\item Projecting to the symmetric tensor algebra $\Sy$ (\Cref{sec:symmetric_tensor_algebra}, \Cref{sec:mc}), i.e., enforcing commutativity, signature moments and signature cumulants turn into the sequence of classical multivariate moments and cumulants of random variable $\bX_{0,T}$.
Projecting the functional equation \eqref{eq:master_sigcum} to $\Sy$ and recasting conditional quadratic variation increments as diamond products (\Cref{def:diamondSym}) we obtain the cumulant equation \eqref{eq:main_with_jumps_sym} originally obtain in \cite{friz2020cumulants}.
This consequence is depicted in Figure~\ref{fig:theorems} by going from the upper-left to the lower-left corner.

\item \Cref{eq:main_with_jumps_sym} should be compared with Riccati's ordinary differential equation from affine process theory \cite{duffie2003affine,cuchiero2011affine, keller2011affine}.
Of course, our results, in particular \cref{eq:master_sigcum,eq:main_with_jumps_sym}, are not restricted to affine semimartingales. In turn, expected signatures and cumulants - and subsequently all our statements above these - require moments, which is not required for the Riccati evolution of the characteristic function of affine processes.
    Of recent interest, explicit diamond expansions have been obtained for
   ``rough affine'' processes, non-Markov by nature, with cumulant generating function characterized by Riccati Volterra equations, see \cite{abijaber2019affine, gatheral2019affine,friz2020cumulants}. 
     It is remarkable that analytic tractability remains intact when one passes to path space and considers signature cumulants (see in particular \cite[Section~6.3]{FHT22}).

\item Projecting \eqref{eq:master_sigcum} to tensor levels, we obtained a \emph{diamond expansion} for multivariate cumulants \eqref{eq:diamond_recursion}.
Such expansions where previously obtained in
\cite{lacoin2019probabilistic, alos2018exponentiation,friz2020cumulants,Fuk21}, together with a range of applications, from quantitative finance (including rough volatility models \cite{abijaber2019affine, gatheral2019affine}), to statistical physics: in \cite{lacoin2019probabilistic} the authors rely on such formulas to compute the cumulant-generating function of log-correlated Gaussian fields, more precisely approximations thereof, that underlies the Sine-Gordon model, which is a key ingredient in their renormalization procedure. 

\item In mathematical finance, a case of particular interest is the $\Sy$-valued semimartingale 
$\hat \bX = (0,aX,b\langle X \rangle, 0, \dotsc)$, for a $d$-dimensional continuous martingale $X$.
The resulting diamond expansion from \cite{friz2020cumulants}, which improves and unifies previous results \cite{lacoin2019probabilistic, alos2018exponentiation}, allowing to calculate joint cumulants of log-price and integrated volatility \cite{friz2022diamonds}.

\item %
We end this summary with a comment on questions of \emph{convergence}: Basic facts of analytic functions show that {\em classical} moment- and cumulant-generating functions, for random variables with finite moments of all orders, have a radius of convergence $\rho \ge 0$,  directly related to growth of the corresponding sequence. Convergence, in the sense $\rho > 0$, implies that the moment problem is well-posed. That is, the moments (equivalently: cumulants) determine the law of the underlying random variable. (Note that this might not be the case if $\rho =0$. See also \cite{friz2020cumulants} for a related discussion in the context of diamond expansions.)
Similarly, understanding the growth of expected signatures or signature cumulants is important task: in a celebrated paper \cite{chevyrev2016characteristic} it was shown that under infinite convergence radius of the expected signature, in the sense of \eqref{eq:geometric_decay}, the ``expected signature problem'' is well-posed; that is, the expected 
signature (equivalently: signature cumulants) determines the law of the random signature. 
In \Cref{sec:levy} we show for the case of time-inhomogeneous Lévy processes how our functional equation lead to a verifiable sufficient condition for expected signature to have an infinite convergence radius.
\end{itemize}

\bibliographystyle{arxivalpha}
\bibliography{library}

\end{document}